%% file: main-neurips-linearmdp.tex
\documentclass{article}
\usepackage[accepted]{aistats2021}

\usepackage[round]{natbib}
     \PassOptionsToPackage{numbers, compress}{natbib}

\usepackage[utf8]{inputenc} 
\usepackage[T1]{fontenc}    
\usepackage{url}            
\usepackage{booktabs}       
\usepackage{amsfonts}       
\usepackage{nicefrac}       
\usepackage{microtype}      
\usepackage[normalem]{ulem}
\usepackage{soul}
\usepackage{cancel}

\input{header1}

\begin{document}
\twocolumn[

\aistatstitle{Learning Infinite-horizon Average-reward MDPs \\ with Linear Function Approximation}

\aistatsauthor{ Chen-Yu Wei \And Mehdi Jafarnia-Jahromi \And Haipeng Luo \And Rahul Jain }

\aistatsaddress{\texttt{chenyu.wei@usc.edu} \And \texttt{mjafarni@usc.edu} \And \texttt{haipengl@usc.edu} \And \texttt{rahul.jain@usc.edu}} 
\vspace*{-15pt}
\aistatsaddress{University of Southern California}

]

\begin{abstract}
We develop several new algorithms for learning Markov Decision Processes in an infinite-horizon average-reward setting with linear function approximation.
Using the optimism principle and assuming that the MDP has a linear structure,
we first propose a computationally inefficient algorithm with optimal $\otil(\sqrt{T})$ regret and another computationally efficient variant with $\otil(T^{\frac{3}{4}})$ regret, where $T$ is the number of interactions.
Next, taking inspiration from adversarial linear bandits, we develop yet another efficient algorithm with $\otil(\sqrt{T})$ regret under a different set of assumptions, improving the best existing result by~\citet{hao2020provably} with $\otil(T^{\frac{2}{3}})$ regret.
Moreover, we draw a connection between this algorithm and the Natural Policy Gradient algorithm proposed by~\citet{kakade2002natural}, and show that our analysis improves the sample complexity bound recently given by \citet{agarwal2019optimality}.
\end{abstract}

\input{intro}
\input{preliminaries}

\input{optimisticLSVI}

\input{EXP2}

\section{Conclusions and Open Problems}
In this work, we provide three new algorithms for learning infinite-horizon average-reward MDPs with linear function approximation, significantly extending and improving previous works.
One key open question is how to achieve the optimal $\otil(\sqrt{T})$ regret efficiently under the linear MDP assumption.  In \pref{app: NPG}, we also discuss another open question related to weakening \pref{ass: assump A4} while maintaining a similar regret bound. 

\section*{Acknowledgement}
HL and CYW are supported by NSF Awards IIS1755781 and IIS-1943607, and a Google Faculty Research Award. RJ and MJ are supported by NSF grants ECCS-1810447, CCF-1817212 and ONR grant N00014-20-1-2258.

\bibliography{online_rl}
\bibliographystyle{plainnat}

\newpage
\onecolumn 
\appendix

\input{appendix-optimisticLSVI}

\input{appendix-exp2}

\input{appendix-NPG}
\end{document}

%% file: header1.tex
\usepackage[algo2e, linesnumbered, vlined, ruled]{algorithm2e}
\usepackage{amssymb}
\usepackage{amsmath} 
\usepackage{amsfonts} 
\usepackage{bbm} 
\usepackage{amsthm} 
\usepackage{xspace}
\usepackage{times}
\usepackage{bm}
\usepackage{algorithm}
\usepackage[noend]{algpseudocode}

\usepackage{color} 

\LinesNumbered

\newcommand{\e}{\mathbf{e}}

\usepackage{graphicx} 
\usepackage{color} 
\usepackage{xcolor}
\usepackage{multirow} 
\usepackage{makecell}
\usepackage[colorlinks=true, linkcolor=blue, citecolor=blue, breaklinks=true]{hyperref}

\newtheorem{theorem}{Theorem}

\newtheorem{lemma}[theorem]{Lemma}

\newtheorem{assumption}{Assumption}

\newtheorem{remark}{Remark}

\DeclareMathOperator*{\argmax}{argmax}
\DeclareMathOperator*{\spn}{sp}

\definecolor{Green}{rgb}{0.13, 0.65, 0.3}
\usepackage[]{color-edits}
\addauthor{HL}{Green}

\newcommand{\de}{\mathrm{d}}
\newcommand{\MU}{\bm{\mu}}
\newcommand{\THE}{\bm{\theta}}
\newcommand{\PHI}{\bm{\phi}}
\newcommand{\spv}{\spn(v^*)}

\newcommand{\Holder}{{H{\"o}lder's inequality}\xspace}
\newcommand{\term}{\textbf{term}}
\newcommand{\algARfull}{\text{Fixed-point OPtimization with Optimism}\xspace}
\newcommand{\algAR}{\textsc{FOPO}\xspace}
\newcommand{\algGamma}{\textsc{OLSVI.FH}\xspace}
\newcommand{\algExp}{\textsc{MDP-Exp2}\xspace}
\newcommand{\exptwo}{\textsc{Exp2}\xspace}
\newcommand{\LSVI}{\textsc{LSVI}\xspace}
\newcommand{\algDis}{\algGamma}

\newcommand{\LineComment}[1]{\hfill$\rhd\ $\text{#1}}

\newcommand{\TV}{\text{\rm TV}}

\newcommand{\PP}{\mathbb{P}}
\newcommand{\Reg}{\text{\rm Reg}}
\newcommand{\tmix}{t_{\text{mix}}}
\newcommand{\trace}[1]{\textsc{tr}\left({#1}\right)}
\newcommand{\one}{\mathbf{1}}
\newcommand{\mix}{t_{\text{mix}}}

\newcommand{\calA}{{\mathcal{A}}}

\newcommand{\calX}{{\mathcal{X}}}

\newcommand{\calF}{\mathcal{F}}
\newcommand{\calV}{\mathcal{V}}
\newcommand{\calN}{\mathcal{N}}

\newcommand{\E}{\mathbb{E}}

\newcommand{\order}{\mathcal{O}}
\newcommand{\otil}{\widetilde{\mathcal{O}}}

\newcommand{\hatQ}{\widehat{Q}}

\newcommand{\thetanew}{\theta_{\text{new}}}
\newcommand{\xaprime}{x_{h'}^{k'}, a_{h'}^{k'}}
\newcommand{\xa}{x_{h}^{k}, a_{h}^{k}}
\newcommand{\Phiprime}{\Phi_{h'}^{k'}}
\newcommand{\rprime}{r_{h'}^{k'}}

\newcommand{\calK}{\mathcal{K}}
\newcommand{\calM}{\mathcal{M}}

\usepackage{prettyref}
\newcommand{\pref}[1]{\prettyref{#1}}

\newcommand{\savehyperref}[2]{\texorpdfstring{\hyperref[#1]{#2}}{#2}}
\newrefformat{eq}{\savehyperref{#1}{Eq. \textup{(\ref*{#1})}}}
\newrefformat{eqn}{\savehyperref{#1}{Equation~\ref*{#1}}}
\newrefformat{lem}{\savehyperref{#1}{Lemma~\ref*{#1}}}
\newrefformat{lemma}{\savehyperref{#1}{Lemma~\ref*{#1}}}
\newrefformat{def}{\savehyperref{#1}{Definition~\ref*{#1}}}
\newrefformat{line}{\savehyperref{#1}{Line~\ref*{#1}}}
\newrefformat{thm}{\savehyperref{#1}{Theorem~\ref*{#1}}}
\newrefformat{corr}{\savehyperref{#1}{Corollary~\ref*{#1}}}
\newrefformat{cor}{\savehyperref{#1}{Corollary~\ref*{#1}}}
\newrefformat{sec}{\savehyperref{#1}{Section~\ref*{#1}}}
\newrefformat{app}{\savehyperref{#1}{Appendix~\ref*{#1}}}
\newrefformat{assump}{\savehyperref{#1}{Assumption~\ref*{#1}}}
\newrefformat{ass}{\savehyperref{#1}{Assumption~\ref*{#1}}}
\newrefformat{ex}{\savehyperref{#1}{Example~\ref*{#1}}}
\newrefformat{fig}{\savehyperref{#1}{Figure~\ref*{#1}}}
\newrefformat{tab}{\savehyperref{#1}{Table~\ref*{#1}}}
\newrefformat{alg}{\savehyperref{#1}{Algorithm~\ref*{#1}}}
\newrefformat{rem}{\savehyperref{#1}{Remark~\ref*{#1}}}
\newrefformat{conj}{\savehyperref{#1}{Conjecture~\ref*{#1}}}
\newrefformat{prop}{\savehyperref{#1}{Proposition~\ref*{#1}}}
\newrefformat{proto}{\savehyperref{#1}{Protocol~\ref*{#1}}}
\newrefformat{prob}{\savehyperref{#1}{Problem~\ref*{#1}}}
\newrefformat{claim}{\savehyperref{#1}{Claim~\ref*{#1}}}
\newrefformat{que}{\savehyperref{#1}{Question~\ref*{#1}}}
\newrefformat{op}{\savehyperref{#1}{Open Problem~\ref*{#1}}}

%% file: intro.tex

\section{Introduction}\label{sec:intro}

Reinforcement learning with value function approximation has gained significant empirical success in many applications. However, the theoretical understanding of these methods is still quite limited.
Recently, some progress has been made for Markov Decision Processes (MDPs) with a transition kernel and a reward function that are both linear in a fixed state-action feature representation (or more generally with a value function that is linear in such a feature representation).
For example, \citet{jin2019provably} develop an optimistic variant of the Least-squares Value Iteration (\LSVI) algorithm~\citep{bradtke1996linear, osband2016generalization} for the finite-horizon episodic setting with regret $\otil(\sqrt{d^3T})$, where $d$ is the dimension of the features and $T$ is the number of interactions.
Importantly, the bound has no dependence on the number of states or actions.

However, the understanding of function approximation for the {\it infinite-horizon average-reward} setting, even under the aforementioned linear conditions, remains underexplored.
Compared to the finite-horizon setting, the infinite-horizon model is often a better fit for real-world problems such as server operation optimization or stock market decision making which last for a long time or essentially never end.
On the other hand, compared to the discounted-reward model, maximizing the long-term average reward also has its advantage in the sense that the transient behavior of the learner does not really matter for the latter case.
Indeed, the infinite-horizon average-reward setting for the tabular case (that is, no function approximation) is a heavily-studied topic in the literature.
Several recent works start to investigate function approximation for this setting, albeit under strong assumptions~\citep{abbasi2019politex, abbasi2019exploration, hao2020provably}.

\renewcommand{\arraystretch}{1.5}
\begin{table*}[t]
\centering
\caption{Summary of our results and comparisons to prior work.
Our first two algorithms are the first results for infinite-horizon average-reward MDPs under Assumptions~\ref{assump: bellman} and~\ref{assump: Linear MDP},
while our third algorithm improves over the best existing results in a setting with a different set of assumptions.
These two set of assumptions are incomparable, in the sense that Assumption~\ref{assump: bellman} is weaker than Assumptions~\ref{ass: uniform mixing} and~\ref{ass: assump A4}, while Assumption~\ref{assump: Linear MDP} is stronger than Assumption~\ref{ass: linear policy value}.
}
\label{tab:main}
\vspace{10pt}
\begin{tabular}{|c|c|c|c|}
\hline
\multirow{2}{*}{\textbf{Algorithm}} & \multirow{2}{*}{\textbf{Regret}}  &  \multicolumn{2}{c|}{\textbf{Assumptions}} \\
\cline{3-4} 
&  &  Explorability  & Structure  \\
\hline
\algAR (\pref{alg: optimistic qlearning}) & $\otil(\sqrt{T})$  & \multirow{2}{*}{\makecell{Bellman optimality equation \\
     (\pref{assump: bellman})}} &  \multirow{2}{*}{\makecell{linear MDP \\ (\pref{assump: Linear MDP})} } \\  
\cline{1-2} 
{\algGamma} (\pref{alg: optimistic qlearning disc})    & $\otil(T^{\frac{3}{4}})$ &  &  \\
\hline
{\algExp} (\pref{alg: mdpexp2})& $\otil(\sqrt{T})$ & \multirow{3}{*}{\makecell{
uniform mixing \\
(\pref{ass: uniform mixing}) \\
uniformly excited features \\
(\pref{ass: assump A4}) 
}}  &  \multirow{3}{*}{\makecell{linear bias function \\
     (\pref{ass: linear policy value})
}} \\
\cline{1-2} 
Politex \citep{abbasi2019politex} & $\otil(T^{\frac{3}{4}})$ & & \\
\cline{1-2}
AAPI \citep{hao2020provably} & $\otil(T^{\frac{2}{3}})$ &  &  \\
\hline 
\end{tabular}
\end{table*}

Motivated by this fact, in this work we significantly expand the understanding of learning MDPs in the infinite-horizon average-reward setting with linear function approximation.
We develop three new algorithms, each with different pros and cons.
Our first two algorithms provably ensure low regret for MDPs with linear transition and reward, which are the {\it first} for this setting to the best of our knowledge.
More specifically, the first algorithm \algARfull (\algAR) is based on the principle of ``optimism in the face of uncertainty'' applied in a novel way.
\algAR aims to find a weight vector (parametrizing the estimated value function) that maximizes the average reward under a fixed-point constraint akin to the \LSVI update involving the observed data and an optimistic term.
The constraint is non-convex and we do not know of a way to efficiently solve it.
\algAR also relies on a lazy update schedule similar to~\citep{abbasi2011improved} for stochastic linear bandits, which is only for the purpose of saving computation in their work but critical for our regret guarantee.
We prove that \algAR enjoys $\otil(\sqrt{d^3T})$ regret with high probability, which is optimal in $T$. (\pref{sec: preliminaries})

Our second algorithm \algGamma addresses the computational inefficiency issue of \algAR with the price of having larger regret.
Specifically, it combines two ideas:
1) solving an infinite-horizon problem via an artificially constructed finite-horizon problem, which is new as far as we know, and
2) the optimistic \LSVI algorithm of~\citet{jin2019provably} for the finite-horizon setting.
\algGamma can be implemented efficiently and is shown to achieve $\otil((dT)^{\frac{3}{4}})$ regret. (\pref{sec:optimism})

Our third algorithm \algExp takes a very different approach and is inspired by another algorithm called MDP-OOMD from~\citet{wei2020model}.
MDP-OOMD runs a particular adversarial multi-armed bandit algorithm for each state to obtain $\otil(\sqrt{T})$ regret (ignoring dependence on other parameters) for the tabular case under an ergodic assumption.
We generalize the idea and apply a particular {\it adversarial linear bandit} algorithm known as \exptwo~\citep{dani2008price, bubeck2012towards} for each state (only conceptually --- the algorithm can still be implemented efficiently).
Under the same set of assumptions made in~\citet{hao2020provably} (which does not necessarily require linear transition and reward), we improve their regret bound from $\otil(T^{\frac{2}{3}})$ to $\otil(\sqrt{T})$. 
In \pref{app: NPG}, we also describe the connection of this algorithm with the Natural Policy Gradient algorithm proposed by~\citet{kakade2002natural}, whose sample complexity bound is recently formalized by~\citet{agarwal2019optimality}. We argue that under the setting considered in \pref{sec: mdpexp2}, their analysis translates to a sub-optimal regret bound of $\otil(T^{\frac{3}{4}})$, and that our improvement over theirs comes from the way we construct the gradient estimates. 


We summarize our results and the comparisons to previous work in \pref{tab:main}.

\paragraph{Related work.}
For the tabular case with finite state and action space in the infinite-horizon average-reward setting, 
the works~\citep{bartlett2009regal, jaksch2010near} are among the first to develop algorithms with provable sublinear regret.
Over the years, numerous improvements have been proposed, see for example~\citep{ortner2018regret, fruit2018efficient, talebi2018variance, fruit2019improved, zhang2019regret, wei2020model}.
In particular, the recent work of~\citet{wei2020model} develops two model-free algorithms for this problem.
We refer the reader to~\citep[Table~1]{wei2020model} for comparisons of existing algorithms. 
As mentioned, our algorithm \algExp is inspired by the MDP-OOMD algorithm of~\citet{wei2020model}.
Also note that their Optimistic Q-learning algorithm reduces an infinite-horizon average-reward problem to a discounted-reward problem.
For technical reasons, we are not able to generalize this idea to the linear function approximation setting (see \pref{sec:FH}). 
Instead, our \algGamma reduces the problem to a finite-horizon version, which is new to the best of our knowledge and might be of independent interest.

The work of \cite{chen2018scalable} considers learning in infinite-horizon average-reward MDPs with linear function approximation, under the assumption that the learner has access to a sampling oracle from which the learner can sample states and actions under any given distribution. The assumptions they make for the MDP is similar to the ones in our \pref{sec: mdpexp2}, and the sample complexity bound they obtain is $\otil\left(1/\epsilon^2 \right)$. However, since the oracle assumption is rather strong, it is not clear how to extend their algorithm to the online setting.

The works of~\citet{abbasi2019politex, abbasi2019exploration, hao2020provably} are among the first to consider the infinite-horizon average-reward setting with function approximation and provable regret guarantees in the online setting.
Their results all depend on some uniformly mixing and uniformly excited feature conditions.
As mentioned, under the same assumption, our \algExp algorithm with $\otil(\sqrt{T})$ regret improves the best existing result by~\citet{hao2020provably} with $\otil(T^{\frac{2}{3}})$ regret.
Moreover, our other two algorithms ensure low regret for linear MDPs without these extra assumptions, which do not appear before.

Provable function approximation has gained growing research interest in other settings as well (finite-horizon or discounted-reward). See recent works~\citep{liu2019neural, wang2019neural, yang2019reinforcement, jin2019provably, zanette2020learning, dong2000sqrt, wang2020provably} for example.
In particular, our \algAR algorithm shares some similarity with the algorithm of~\citet{zanette2020learning}, which also relies on solving an optimization problem under a constraint akin to \LSVI, with no efficient implementation.

Adversarial linear bandit is also known as bandit linear optimization.
The \exptwo algorithm~\citep{bubeck2012towards}, on top of which our \algExp algorithm is built, is also known as Geometric Hedge~\citep{dani2008price} or ComBand~\citep{cesa2012combinatorial} in the literature. A concurrent work by~\citet{neu2020online} proposes an algorithm called \textsc{MDP-LinExp3} for the linear function approximation setting that is also based on the adversarial linear bandit framework. However, their result is incomparable to ours because they focus on finite-horizon MDPs with adversarial reward, and they assume that the learner has access to a sampling oracle. 

%% file: preliminaries.tex

\section{Preliminaries}
\label{sec: preliminaries}

We consider infinite-horizon average-reward Markov Decision Processes (MDPs) described by $(\calX, \calA, r, p)$ where $\calX$ is a Borel state space with possibly infinite number of elements, $\calA$ is a finite action set, $r: \calX \times \calA \to [-1, 1]$ is the (unknown) reward function, and $p(\cdot | x, a)$ is the (unknown) transition kernel induced by $x,a$, satisfying $\int_{\calX} p(\de x' | x,a) = 1$ (following integral notation from~\citet{hernandez2012adaptive}). 

The learning protocol is as follows.
A learner interacts with the MDP through $T$ steps, starting from an arbitrary initial state $x_1 \in \calX$.
At each step $t$, the learner decides an action $a_t$, and then observes the reward $r(x_t, a_t)$ as well as the next state $x_{t+1}$ which is a sample drawn from $p(\cdot | x_t, a_t)$. 
The goal of the learner is to be competitive against any fixed stationary policy.
Specifically, a stationary policy is a mapping 
$\pi: \calX\rightarrow \Delta_\calA$ with $\pi(a|x)$ specifying the probability of selecting action $a$ at state $x$.
The long-term average reward of a stationary policy $\pi$ starting from state $x\in\calX$ is naturally defined as: 
\begin{align*}
J^\pi(x) \triangleq \liminf_{T \to \infty} \frac{1}{T} \E \Bigg[
&\sum_{t=1}^T r(x_t, a_t) \Bigm\vert x_1 = x, \;\;\forall t \geq 1, \\
&  a_t\sim \pi(\cdot|x_t), \;  x_{t+1}\sim p(\cdot|x_t, a_t)  \Bigg].
\end{align*}
The performance measure of the learner, known as regret, is then defined as $\Reg_T := \max_{\pi} \sum_{t=1}^T (J^\pi(x_1) - r(x_t, a_t))$, which is the difference between the total rewards of the best stationary policy and that of the learner.

However, in contrast to the finite-horizon episodic setting where ensuring sublinear regret is always possible, it is known that in our setting a necessary condition is that the optimal policy has a long-term average reward that is independent of the initial state~\citep{bartlett2009regal}.
To this end, throughout the paper we only consider a broad subclass of MDPs where a certain form of Bellman optimality equation holds~\citep{hernandez2012adaptive}: 
\begin{assumption}[Bellman optimality equation]
     \label{assump: bellman}
There exist $J^*\in\mathbb{R}$ and bounded measurable functions $v^*:\calX\rightarrow \mathbb{R}$ and $q^*: \calX\times \calA\rightarrow \mathbb{R}$ such that the following holds for all $x \in \calX$ and $a\in\calA$:
\begin{equation}\label{eq: bellman}
\begin{split}
     J^* + q^*(x,a) &= r(x,a) + \E_{x'\sim p(\cdot|x,a)} [v^*(x')], \\
     v^*(x) &= \max_{a\in\calA}q^*(x,a).  
\end{split}     
\end{equation}
\end{assumption}

Indeed, under this assumption, the claim is that a policy $\pi^*$ that deterministically selects an action from $\argmax_a q^*(x,a)$ at each state $x$ is the optimal policy, with $J^{\pi^*}(x) = J^*$ for all $x$.
To see this, note that for any policy $\pi$, using the Bellman optimality equation we have
\begin{align*}
    &J^{\pi}(x)=\liminf_{T\rightarrow \infty}\frac{1}{T}\E\Bigg[\sum_{t=1}^T \Bigg(J^* + \sum_{a\in\calA} q^*(x_t, a)\cdot \pi(a|x_t) \\
    &\qquad\qquad\qquad\qquad\qquad\qquad - v^*(x_{t+1})\Bigg)\Bigg]  \\
    &\leq \liminf_{T\rightarrow \infty}\frac{1}{T}\E\Bigg[\sum_{t=1}^T \left(J^* + v^*(x_t) - v^*(x_{t+1})\right)\Bigg]=J^*,
\end{align*}
with equality attained by $\pi^*$, proving the claim.
Consequently, under \pref{assump: bellman} we simply write the regret as $\Reg_T := \sum_{t=1}^T (J^* - r(x_t, a_t))$.



All existing works on regret minimization for infinite-horizon average-reward MDPs make this assumption, either explicitly or through even stronger assumptions which imply this one.
In the tabular case with a finite state space, weakly communicating MDPs is the broadest class to study regret minimization in the literature, and is known to satisfy \pref{assump: bellman} (see~\citep{puterman2014markov}). 
More generally, \pref{assump: bellman} holds under many other common conditions; see~\citep[Section~3.3]{hernandez2012adaptive}.

Note that $v^*(x)$ and $q^*(x,a)$ quantify the \emph{relative advantage} of starting with $x$ and starting with $(x,a)$ respectively and then acting optimally in the MDP. 
Therefore, $v^*$ is sometimes called the \emph{state bias function} and $q^*$ is called the \emph{state-action bias function}. 

For a bounded function $v: \calX\rightarrow \mathbb{R}$, we define its span as $\spn(v) \triangleq \sup_{x,x'\in\calX} |v(x)-v(x')|$. 
Notice that if $(v^*, q^*)$ is a solution of \pref{eq: bellman}, then a translated version $(v^*-c, q^*-c)$ for any constant $c$ is also a solution.  In the remaining of the paper,
we let $(v^*, q^*)$ be an arbitrary solution pair of \pref{eq: bellman} with a small span $\spv$ in the sense that $\spv \leq 2\spn(v')$ for any other solution $(v', q')$.  We also assume without loss of generality $|v^*(x)|\leq \frac{1}{2}\spv$ for any $x$ because we can perform the above translation and center the values of $v^*$ around zero. 
Similarly to previous works (e.g.~\citep{wei2020model}), $\spv$ is assumed to be known to the learner.

%% file: optimisticLSVI.tex
\section{Optimism-based Algorithms}
\label{sec:optimism}

In this section, we present two optimism-based algorithms with sublinear regret, under only one extra assumption that the MDP is {\it linear} (also known as low-rank MDPs).
We emphasize that earlier works for linear MDPs in the finite-horizon average-reward setting all require extra strong assumptions~\citep{abbasi2019politex, abbasi2019exploration, hao2020provably}.

Specifically, a linear MDP has a transition kernel and a reward function both linear in some state-action feature representation, formally summarized as:

\begin{assumption}[Linear MDP]
    \label{assump: Linear MDP}
    There exist a known $d$-dimensional feature mapping $\Phi: \calX\times \calA\rightarrow \mathbb{R}^d$, $d$ unknown measures $\MU=(\mu_1, \mu_2, \ldots, \mu_d)$ over $\calX$, and an unknown vector $\THE \in \mathbb{R}^d$ such that for all $x,x'\in\calX$ and $a\in\calA$,
    \begin{align*}
         p(x'~|~x,a) = \Phi(x,a)^\top \MU(x'), \qquad r(x,a) = \Phi(x,a)^\top \THE. 
    \end{align*}
Without loss of generality, we further assume that for all $x\in\calX$ and $a\in\calA$, $\|\Phi(x,a)\|\leq \sqrt{2}$, the first coordinate of $\Phi(x,a)$ is fixed to $1$, and that $\|\MU(\calX)\|\leq \sqrt{d}$, $\|\THE\|\leq \sqrt{d}$, where we use $\MU(\calX)$ to denote the vector $(\mu_1(\calX), \ldots, \mu_d(\calX))$ and $\mu_i(\calX)\triangleq \int_{\calX} \de\mu_i(x)$ is the total measure of $\calX$ under $\mu_i$. (All norms  are 2-norm.)
\end{assumption}

In \citep{jin2019provably}, the same assumption is made except for a different rescaling:  $\|\Phi(x,a)\|\leq 1$, $\|\MU(\calX)\|\leq \sqrt{d}$, and $\|\THE\|\leq \sqrt{d}$.
The reason that this is without loss of generality is not justified in \citep{jin2019provably}, and for completeness we prove this in \pref{app: MVEE section}.
With this scaling, clearly one can augment the feature $\Phi(x,a)$ with a constant coordinate of value $1$ and augment $\MU(x)$ and $\THE$ with a constant coordinate of value $0$, such that the linear structure is preserved while the scaling specified in \pref{assump: Linear MDP} holds.



Under \pref{assump: Linear MDP}, one can show that the state-action bias function $q^*$ is in fact also linear in the features.
\begin{lemma}\label{lem: w* existence}
     Under \pref{assump: bellman} and \pref{assump: Linear MDP}, there exists a fixed weight vector $w^* \in \mathbb{R}^d$ such that $q^*(x,a) = \Phi(x,a)^\top w^*$ for all $x\in \calX$ and $a\in \calA$, and furthermore, $\|w^*\|\leq (2 + \spv)\sqrt{d}$.  
\end{lemma}

Based on this lemma, a natural idea emerges: at time $t$, build an estimator $w_t$ of $w^*$ using observed data, then act according to the estimated long-term reward of each action given by $\Phi(x_t,a)^\top w_t$.
While the idea is intuitive, how to construct the estimator and, perhaps more importantly, how to incorporate the optimism principle well known to be important for learning with partial information, are highly non-trivial.
In the next two subsections, we describe two different ways of doing so, leading to our two algorithms \algAR and \algGamma.

\subsection{\algARfull (\algAR)}
\label{sec: optimistic q-learning with function approximation}

\begin{algorithm}[t]

\caption{\algARfull (\algAR)}
\label{alg: optimistic qlearning}
\textbf{Parameters}: $0<\delta<1$, $\lambda=1$, $\beta=20(2+\spv)d\sqrt{\log(T/\delta)}$\\
\textbf{Initialize}: 
$\Lambda_1 = \lambda I$ where $I \in \mathbb{R}^{d\times d}$ is the identity matrix \\ 
 \For{$t=1, \ldots, T$}{
   \If{$t=1$ or $\det(\Lambda_t)\geq 2\det(\Lambda_{s_{t-1}})$}{ 
    Set $s_t= t$ \LineComment{$s_t$ records the most recent update} \\
    Let $w_t$ be the solution of the optimization problem:  
    \begin{align}
         &\max_{w_t, b_t \in \mathbb{R}^d,   J_t \in \mathbb{R}} \;\; J_t \notag \\
         \textit{s.t.}\ \ & w_t = \Lambda_{t}^{-1} \sum_{\tau=1}^{t-1}\Big(\Phi(x_\tau, a_\tau)(r(x_\tau, a_\tau)  \label{eq:fixed_point}\\
         &\qquad\qquad\qquad\qquad - J_t + v_t(x_{\tau+1}) ) + b_t\Big) \notag \\
         & q_t(x,a) = \Phi(x,a)^\top w_t, \quad v_t(x) = \max_a q_t(x,a) \notag \\
         & \|b_t\|_{\Lambda_t} \leq \beta, \quad \|w_t\|\leq (2+\spv)\sqrt{d} \notag
    \end{align} 
    }
    \Else{
             \vspace{-0.8cm}
         \begin{align*}
             &(w_t, J_t, b_t, v_t, q_t, s_t) \\
             &= (w_{t-1}, J_{t-1}, b_{t-1}, v_{t-1}, q_{t-1}, s_{t-1})
             \end{align*} 
    }
    Play $a_t=\argmax_{a} q_{t}(x_t, a)$ \\
    Observe $r(x_t, a_t)$ and $x_{t+1}$ \\
    Update $\Lambda_{t+1} =  \Lambda_t + \Phi(x_{t}, a_{t})\Phi(x_{t}, a_{t})^\top$ 
} 
\end{algorithm}

We present our first algorithm \algAR which is computationally inefficient but achieves regret $\otil(\spv\sqrt{d^3T})$.
This is optimal in $T$ since even in the tabular case $\order(\sqrt{T})$ is unimprovable~\citep{jaksch2010near}.
See \pref{alg: optimistic qlearning} for the complete pseudocode.

As mentioned, the key part lies in how the estimator $w_t$ is constructed.
In \pref{alg: optimistic qlearning}, this is done by solving an optimization problem over certain constraints.
To understand the first constraint Eq.~\eqref{eq:fixed_point},
recall that $q^*(x,a)=\Phi(x,a)^\top w^*$ satisfies the Bellman optimality equation: 
\begin{align*}
     &\Phi(x,a)^\top w^* = r(x,a) - J^* + \int_\calX v^*(x')p(\de x'~|~x,a)  \\
      &= r(x,a) - J^*  + \int_\calX \left(\max_{a'} \Phi(x',a')^\top w^*\right) p(\de x'~|~x,a) . 
\end{align*}
While $p$ and $r$ are unknown, we do observe samples $x_1, \ldots, x_{t-1}$ and $r(x_1, a_1), \ldots, r(x_{t-1}, a_{t-1})$.
If for a moment we assume $J^*$ was known,
then it is natural to try to find $w_t$ such that $\forall \tau = 1, \ldots, t-1$,
\begin{align}
     \Phi(x_\tau,a_\tau)^\top w_t \approx r(x_\tau, a_\tau) - J^* + \max_{a'}\Phi(x_{\tau+1},a')^\top w_t. \label{eq:MSE formula}
\end{align}
In common variants of Least-squares Value Iteration (\LSVI) update, 
the $w_t$ on the right hand side of \pref{eq:MSE formula} would be replaced with another already computed weight vector $w_t'$ that is either from the last iteration (i.e, $w_{t-1}$) or from the next layer in the case of episodic MDPs.
Then solving a least-squares problem with regularization $\lambda\|w_t\|^2$ gives a natural estimate of $w_t$:
\[
\Lambda_{t}^{-1} \sum_{\tau=1}^{t-1}\Phi(x_\tau, a_\tau)\left(r(x_\tau, a_\tau) - J^* + \max_{a'}\Phi(x_{\tau+1},a')^\top w_t' \right) 
\]
where $\Lambda_t = \lambda I + \sum_{\tau<t} \Phi(x_{\tau}, a_{\tau})\Phi(x_{\tau}, a_{\tau})^\top$ is the empirical covariance matrix.
Based on this formula, what we propose in \pref{alg: optimistic qlearning} are the following three modifications.
First, instead of using an already computed weight $w_t'$, we directly set it back to $w_t$ (and thus $\max_{a'}\Phi(x_{\tau+1},a')^\top w_t' = v_t(x_{\tau+1})$), making the formula a fixed-point equation now.
Second, to incorporate uncertainty, we introduce a slack variable $b_t$ with a bounded quadratic norm $\|b_t\|_{\Lambda_t} \triangleq \sqrt{b_t^\top\Lambda_t b_t} \leq \beta$ (for a parameter $\beta$) that controls the amount of uncertainty.
Last, to deal with the fact that $J^*$ is unknown, we replace it with a variable $J_t$ (arriving at \pref{eq:fixed_point} finally), and apply the well-known principle of \emph{optimism in the face of uncertainty} --- we maximize the long-term average reward $J_t$ (over $w_t, b_t$ and $J_t$) under the aforementioned constraints and also $\|w_t\|\leq (2+\spv)\sqrt{d}$ in light of \pref{lem: w* existence}.

With the vector $w_t$ and the corresponding bias function $q_t$, the algorithm simply plays $a_t=\argmax_{a} q_{t}(x_t, a)$ greedily.
Note that $w_t$ is only updated when the determinant of $\Lambda_t$ doubles compared to that of $\Lambda_{s_{t-1}}$ where $s_{t-1}$ is the time step with the most recent update before time $t$. 
This can happen at most $\order(d\log T)$ times.
Similar ideas are used in e.g.,~\citep{abbasi2011improved} for stochastic linear bandits.
However, while they use this lazy update only to save computation, here we use it to make sure that $w_t$ does not change too often, which is critical for our regret analysis.

We point out that the closest existing algorithm we are aware of is the one from a recent work~\citep{zanette2020learning} for the finite-horizon setting.
Just like theirs, our algorithm also does not admit an efficient implementation due to the complicated nature of the optimization problem.
However, it can be shown that the constraint set is non-empty with $(w_t, b_t, J_t) = (w^*, b, J^*)$ for some $b$ being a feasible solution (with high probability).
This fact also immediately implies that $J_t$ is indeed an optimistic estimator of $J^*$ in the following sense:
\begin{lemma}
    \label{lemma: J optimistic}
    With probability at least $1-\delta$,  \pref{alg: optimistic qlearning} ensures $J_t \geq J^*$ for all $t$.
\end{lemma}
With the help of this lemma, we prove the following regret bound of \algAR with optimal (in $T$) rate.

\begin{theorem}
     \label{thm: optimistic AR}
     Under Assumptions~\ref{assump: bellman} and~\ref{assump: Linear MDP}, \algAR guarantees with probability at least $1-3\delta$: \[\Reg_T = \order\left(\spv\log (T/\delta) \sqrt{d^3T }\right).\] 
\end{theorem}

\subsection{Finite-Horizon Optimistic Least-Square Value Iteration (\algDis)}
\label{sec:FH}



Next, we present another optimism-based algorithm which can be implemented efficiently, albeit with a suboptimal regret guarantee.
The high-level idea is still based on \LSVI. However, since we do not know how to efficiently solve a fixed-point problem as in \pref{alg: optimistic qlearning}, we ``open the loop'' by solving a finite-horizon problem instead. 
More specifically, we divide the $T$ rounds into $\nicefrac{T}{H}$ episodes each with $H$ rounds, and run a finite-horizon optimistic \LSVI algorithm over the episodes as in~\citep{jin2019provably}. 

\DontPrintSemicolon 
\begin{algorithm}[t]
\setcounter{AlgoLine}{0}
\caption{\algDis}
\label{alg: optimistic qlearning disc}
\textbf{Parameters}: $0<\delta<1$,  $\lambda=1$,  $\beta=40dH\sqrt{\log (T/\delta)}$, $H =\max\left\{ \frac{\sqrt{\spv} T^{1/4}}{d^{3/4}},  \left(\frac{\spv T}{d^2}\right)^{\nicefrac{1}{3}}\right\}$  \\
\textbf{Initialize}:
$\Lambda_1 = \lambda I$ where $I \in \mathbb{R}^{d\times d}$ is the identity matrix \\ 
\textbf{Define}: $x_{h}^k = x_t$ and $a_{h}^k = a_t$, for $t = (k-1)H+h$ \\
\nl \For{$k=1, \ldots, \nicefrac{T}{H}$}{
\nl      Define $V_{H+1}^k(x)=0$ for all $x$. \\
\nl      \For{$h=H, \ldots, 1$}{    \label{line: backward ind}
\nl         Compute 
\begin{align*}
w_{h}^k &= \Lambda_{k}^{-1}\sum_{k'=1}^{k-1}\sum_{h'=1}^H \Phi(x_{h'}^{k'}, a_{h'}^{k'})\Big(r(x_{h'}^{k'}, a_{h'}^{k'}) \\
&\qquad\qquad\qquad\qquad+ V_{h+1}^k (x_{h'+1}^{k'})\Big)
\end{align*} \label{line: assignment step}  \\
\nl          Define 
\begin{align*}
\hatQ_{h}^k(x,a) &= w_{h}^k \cdot \Phi(x,a) + \beta \sqrt{\Phi(x,a)^\top \Lambda_k^{-1}\Phi(x,a)} \\
Q_h^k(x,a) &= \min\left\{\hatQ_h^k(x,a), H\right\} \\
V_{h}^k(x) &= \max_a Q_{h}^k(x,a)
\end{align*} \label{line: backward ind end}   
      }
\nl      \For{$h=1,\ldots, H$}{
\nl           Play $a_{h}^k=\argmax_a  Q_{h}^k(x_h^k,a)$ \label{line: greedy in finite horizon}  \\
\nl           Observe $x_{h}^k$ and $r(x_h^k,a_h^k)$  
      }
\nl     Update $\Lambda_{k+1} = \Lambda_k + \sum_{h=1}^H \Phi(x_h^k,a_h^k)\Phi(x_h^k,a_h^k)^\top $\label{line: Lambda update} 
}
 
\end{algorithm}


The resulted algorithm is shown in \pref{alg: optimistic qlearning disc}. For simplicity, we replace the time index $t$ with a combination of an \emph{episode index} $k$ and a \emph{step index} $h$ within the episode. This gives the relation $t=(k-1)H+h$, and $(x_t,a_t)$ is written as $(x^k_h, a^k_h)$. At the beginning of each episode $k$, the learner computes a set of Q-function parameters $w^k_1, \ldots, w^k_H$ by backward calculation using all historical data (\pref{line: backward ind} to \pref{line: backward ind end}). 
Note that \pref{line: assignment step} is now simply an assignment step (as opposed to a fixed-point problem) since $V_{h+1}^k$ is computed already when in step $h$.
In \pref{line: backward ind end}, we introduce optimism by incorporating a bonus term $\beta\|\Phi(x,a)\|_{\Lambda_k^{-1}}$ into the definition of $\hatQ_h^k(x,a)$, and hence $Q_h^k(x,a)$. Then in step $h$ of episode $k$, the learner simply follows the greedy choice suggested by $Q_h^k(x^k_h,\cdot)$ (\pref{line: greedy in finite horizon}).

Note that \pref{alg: optimistic qlearning disc} is slightly different from the version in~\citep{jin2019provably}: they maintain a different covariance matrix $\Lambda^k_h$ separately for each step $h$, but we only maintain a single $\Lambda_k$ for all $h$. 
Similarly, their $w_{h}^k$ is computed using only data related to step $h$ from all previous episodes, while ours is computed using all previous data.
This is because in our problem, the steps within an episode share the same transition and reward functions, and consequently they can be learned jointly, which eventually reduces the sample complexity. 


Clearly, this reduction ensures that the learner has low regret against the best policy for the {finite-horizon} problem that we create. However, since our original problem is about average-reward over infinite horizon, we need to argue that the best finite-horizon policy also performs well under the infinite-horizon criteria. Indeed, we show that the sub-optimality gap of the best finite-horizon policy is bounded by some quantity governed by $\spv /H$, which is intuitive since the larger $H$ is, the smaller the gap becomes (see \pref{lemma: connecting lemma}).

In our analysis, for a fixed episode we define $\pi=(\pi_1, \ldots, \pi_H)$ as the finite-horizon policy (i.e., a length-$H$ sequence of policies), where each $\pi_h$ is a mapping $\calX\rightarrow \Delta_{\calA}$. 
For any such finite-horizon policy $\pi$, we define $Q^{\pi}_h(x,a)$ and $V^{\pi}_h(x)$ as the value functions for the finite-horizon problem we create, which satisfy: $V_{H+1}^{\pi}(x)=0$ and for $h=H,\ldots, 1$,  
\begin{equation}
\begin{split}
     Q_h^\pi(x,a) &= 
          r(x,a) +  \E_{x'\sim p(\cdot|x,a)}[V_{h+1}^\pi(x')], \\
     V_h^\pi(x) &= \E_{a\sim \pi_h(\cdot|x)} Q_h^\pi(x,a). 
\end{split}\label{eq:discounted_bellman}
\end{equation}
The analysis of the algorithm relies on the following key lemma, which shows that $Q^k_h(x,a)$ upper bounds $Q_h^{\pi}(x,a)$ for any $\pi$. 

\begin{lemma}\label{lem: optimistic Q}
    With probability at least $1-\delta$, \pref{alg: optimistic qlearning disc} ensures for any finite-horizon policy $\pi$ that $\forall x,a,k,h$. 
    \begin{align*}
    0&\leq Q^k_h(x,a)- Q_h^{\pi}(x,a) \\
    &\leq \E_{x'\sim p(\cdot|x,a)} \left[V_{h+1}^k(x')-V_{h+1}^{\pi}(x')\right] + 2\beta\|\Phi(x,a)\|_{\Lambda_k^{-1}}.
    \end{align*} 
\end{lemma}

With the help of \pref{lem: optimistic Q}, we prove the final regret bound of \algGamma stated in the next theorem
(proof deferred to the appendix).
\begin{theorem}
     \label{thm: discounted main thm}
     Under Assumptions~\ref{assump: bellman} and~\ref{assump: Linear MDP}, \algDis guarantees with probability at least $1-3\delta$: \[\Reg_T =\otil\left(\sqrt{\spv  }(dT)^{\frac{3}{4}} + \left(\spv  d T \right)^{\frac{2}{3}}\right).\] 
\end{theorem}
Note that although our bound is suboptimal, \algDis is the first efficient algorithm with sublinear regret for this setting under only Assumptions~\ref{assump: bellman} and~\ref{assump: Linear MDP}.

%% file: EXP2.tex
\section{The \algExp Algorithm}
\label{sec: mdpexp2}

There are two disadvantages of the optimism-based algorithms introduced in the last section.
First, they require the transition kernel and reward function to be both linear in the feature (\pref{assump: Linear MDP}), which is restrictive and might not hold especially when $d$ is small.
Second, even for the polynomial-time algorithm \algDis, it is still computationally intensive because in \pref{line: assignment step} of the algorithm, $V_{h+1}^k$ is applied to all previous states, and every evaluation of $V_{h+1}^k$ requires computing $\|\Phi(x,a)\|_{\Lambda_k}$. Since this is done for every $k$, the total computational cost of the algorithm is super-linear in $T$.
In fact, all existing optimism-based algorithms with linear function approximation
suffer the same issue~\cite{yang2019reinforcement, jin2019provably, zanette2020learning}.

To this end, we propose yet another algorithm based on very different ideas.
It is computationally less intensive and it enjoys $\otil(\sqrt{T})$ regret, albeit under a different (and non-comparable) set of assumptions compared to those in \pref{sec:optimism}.
Note that these are the same assumptions made in~\citep{abbasi2019politex, hao2020provably}.
Below, we start with stating these assumptions, followed by the description of our algorithm.

The first assumption we make is that the MDP is uniformly mixing.
\begin{assumption}[Uniform Mixing]
     \label{ass: uniform mixing}
     There exists a constant $\tmix\geq 1$ such that for any policy $\pi$, and any distributions $\nu_1, \nu_2\in\Delta_\calX$ over the state space, 
     \begin{align*}
          \left\| \PP^{\pi}\nu_1 -  \PP^{\pi}\nu_2 \right\|_\TV \leq e^{-1/\tmix}\|\nu_1-\nu_2\|_\TV,  
     \end{align*} 
     where $(\PP^\pi\nu)(x') = \int_{\calX} \sum_{a\in\calA} \pi(a|x) p(x'|x,a) \de \nu(x)$ and $\|\cdot\|_\TV$ is the total variation.
\end{assumption}
Under this uniform mixing assumption, we are able to define the stationary state distribution under a policy $\pi$ as $\nu^\pi=\left(\PP^{\pi}\right)^\infty \nu_1$ for an arbitrary initial distribution $\nu_1$. 
Also, now we not only have the Bellman optimality equation \eqref{eq: bellman}  (that is, \pref{ass: uniform mixing} implies \pref{assump: bellman}),
but also a Bellman equation for every policy $\pi$, as shown in the following lemma.
\begin{lemma}
      \label{lem:uniformly mixing lemma}
      Suppose \pref{ass: uniform mixing} holds. For any $\pi$, its long-term average reward $J^\pi(x)$ is independent of the initial state $x$, thus denoted as $J^{\pi}$. Also, the following {Bellman equation} holds: 
      \begin{align*}
           J^{\pi} + q^\pi(x,a) &= r(x,a) + \E_{x'\sim p(\cdot|x,a)}[v^{\pi}(x')], \\
           v^\pi(x) &= \sum_{a\in\calA} \pi(a|x) q^{\pi}(x,a)
      \end{align*}
      for some measurable functions $v^\pi: \calX \rightarrow [-4\tmix, 4\tmix]$ and $q^\pi: \calX\times \calA\rightarrow [-6\tmix, 6\tmix]$ with $\int_\calX v^\pi(x) \de\nu^\pi(x) = 0$.
\end{lemma}

On the other hand, with this assumption (stronger than \pref{assump: bellman}), 
we can replace \pref{assump: Linear MDP} (linear MDP) with the following weaker one that only requires the bias function $q^\pi$ to be linear. 

\begin{assumption}[Linear bias function]
     \label{ass: linear policy value}
     There exists a known $d$-dimensional feature mapping $\Phi: \calX\times \calA\rightarrow \mathbb{R}^d$ such that for every policy $\pi$, $q^{\pi}(x,a)$ can be written as $\Phi(x,a)^\top w^{\pi}$ for some weight vector $w^\pi \in \mathbb{R}^d$.
     Again, without loss of generality (justified in \pref{app: MVEE section}), we assume that for all $x,a$, $\|\Phi(x,a)\|\leq \sqrt{2}$ holds, the first coordinate of $\Phi(x,a)$ is fixed to $1$, and for all $\pi$, $\|w^{\pi}\|\leq 6\mix\sqrt{d}$. 
\end{assumption}

In \pref{lemma:linear MDP implies linear value} in the appendix, we show that this is indeed weaker than the linear MDP assumption.
Note that there are indeed practical examples where \pref{ass: linear policy value} holds but \pref{assump: Linear MDP} does not (see the queueing network example of~\citep{de2003linear}).


The last assumption we make is uniformly excited features, which intuitively guarantees that every policy is explorative in the feature space. 
\begin{assumption}[Uniformly excited features]
\label{ass: assump A4}
    There exists $\sigma>0$ such that for any $\pi$, 
    \begin{align*}
        \lambda_{\min}  \left(\int_{\calX}\left( \sum_{a}\pi(a|x)\Phi(x,a)\Phi(x,a)^\top\right) \de\nu^{\pi}(x)\right)\geq \sigma,
    \end{align*}
    where $\lambda_{\min}$ denotes the smallest eigenvalue.
\end{assumption}

This assumption is needed due to the nature of our algorithm that only performs local search of the parameters.
It can potentially be weakened if we combine our algorithm with the idea of~\citet{abbasi2019exploration} (details omitted). 

\subsection{Algorithm and guarantees}
We are now ready to present our \algExp algorithm, shown in \pref{alg: mdpexp2}. 
It extends the idea of running an adversarial bandit algorithm at each state from the tabular case~\cite{neu2013online, wei2020model} to the continuous state case, by using an adversarial linear bandit algorithm \exptwo~\cite{bubeck2012towards}.

\begin{algorithm}[t]
\setcounter{AlgoLine}{0}
    \caption{\algExp}
    \label{alg: mdpexp2}
    \textbf{Parameter}: $N=8\tmix \log T$, $B=32N \log (dT)\sigma^{-1}$, $\eta=\min\left\{\sqrt{1/(T\tmix)}, \sigma/(24N)\right\}$.  
  
\nl    \For(\LineComment{$k$ indexes an epoch}){$k=1, \ldots, \nicefrac{T}{B}$}{
\nl        Define policy $\pi_k$ such that for every $x \in \calX$: 
\[\pi_k(a|x) \propto \exp\left(\eta \sum_{j=1}^{k-1} \Phi(x,a)^\top w_j\right)\] \label{line:exp} \\
\nl Execute $\pi_k$ in the entire epoch: \\
 \nl       \For{$t=(k-1)B+1, \ldots, kB$}{
\nl            Play $a_t \sim \pi_k(\cdot|x_t)$, observe $r_t(x_t,a_t)$ and $x_{t+1}$ 
        }
 \nl        \For(\LineComment{$m$ indexes a trajectory}){$m=1, \ldots, \nicefrac{B}{2N}$}{
 \nl               Define \[\tau_{k,m} = (k-1)B+2N(m-1) + N +1,\] the first step of the $m$-th trajectory \\
\nl                Compute \[R_{k,m} = \sum_{t=\tau_{k,m}}^{\tau_{k,m}+N-1}r(x_t,a_t),\] the total reward of the $m$-th trajectory
         }
        \ \\
\nl        Compute 
        \begin{align*}
            M_k &= \sum_{m=1}^{\frac{B}{2N}} \sum_{a} \pi_k(a|x_{\tau_{k,m}})\Phi(x_{\tau_{k,m}},a)\Phi(x_{\tau_{k,m}},a)^\top, 
        \end{align*}
\nl   \If{$\lambda_{\min}(M_k)\geq  \frac{B\sigma}{24N}$}{   \label{line: check lambda min}
Set $w_k = M_k^{-1}\sum_{m=1}^{\frac{B}{2N}} \Phi(x_{\tau_{k,m}}, a_{\tau_{k,m}})R_{k,m}$
}      
\Else{Set $w_k =\bm{0}$ }    
         \label{line:estimator}
    }
\end{algorithm}

Specifically, \algExp proceeds in \emph{epochs} of equal length $B=\otil(d\tmix/\sigma)$. 
In each epoch $k$, the algorithm executes a fixed policy $\pi_k$ (explained later), 
and collects $\frac{B}{2N}$ disjoint trajectories, each of length $N=\otil(\tmix)$.
Between every two consecutive trajectories, there is a window of length $N$ in which the algorithm does not collect any samples, so that the correlation of samples from different trajectories is reduced.
See \pref{fig: data collection} in the appendix for an illustration.

In the analysis, we show that the expected total reward of a trajectory is roughly $q^{\pi}(x_\tau,a_\tau) + NJ^\pi$ (\pref{lem: closeness between Rkm and q}), where $\pi$ is the policy used to collect that trajectory and $\tau$ is the first step of the trajectory. 
By \pref{ass: linear policy value} we have $q^{\pi}(x_\tau,a_\tau) + NJ^\pi=\Phi(x_\tau, a_\tau)^\top \left(w^\pi + NJ^{\pi}\e_1\right)$. 
This observation allows us to draw a connection between this problem and adversarial linear bandits.
To see this,
first note that the regret is roughly $B\sum_{k=1}^{T/B} (J^*-J^{\pi_k})$. 
By the standard value difference lemma~\citep[Lemma~5.2.1]{kakade2003sample}, we have 
\begin{align*}
     &\sum_{k=1}^{T/B} \left(J^*-J^{\pi_k}\right) =\\
     & \int_{\calX} \left( \sum_{k=1}^{T/B}\sum_{a}\left( \pi^*(a|x) - \pi_k(a|x) \right)q^{\pi_k}(x,a)\right)\de \nu^{\pi^*}(x) 
\end{align*}
where according to the previous observation and the fact $\sum_a (\pi^*(a|x)-\pi_k(a|x))NJ^{\pi_k}=0$, the term in the parentheses with respect to a fixed state $x$ can be further written as
$\sum_{k=1}^{T/B}\sum_{a}\left( \pi^*(a|x) - \pi_k(a|x) \right) \Phi(x,a)^\top \left(w^{\pi_k} + NJ^{\pi_k}\e_1\right)$.
This is exactly the regret of a standard online learning problem over a set of actions $\{\Phi(x,a)\}_{a\in\calA}$ with linear reward functions parameterized by a weight vector $(w^{\pi_k} + NJ^{\pi_k}\e_1)$ at step $k$. 
Moreover, since we do not observe this weight but have access to the reward of a trajectory whose mean is roughly $\Phi(x,a)^\top \left(w^{\pi_k} + NJ^{\pi_k}\e_1\right)$ as mentioned, 
we are in the so-called bandit setting. 
In fact, since the weight can generally change arbitrarily over time (because $\pi_k$ is changing), this is an adversarial linear bandit problem.

With this connection in mind, the idea behind \algExp is clear --- it conceptually runs a variant of the linear bandit algorithm \exptwo for each state.
Specifically, in epoch $k$ the algorithm constructs an estimator $w_k$ for the reward vector $w^{\pi_k} + NJ^{\pi_k}\e_1$. 
The construction mostly follows the idea of \exptwo, with the only difference being the way of controlling the variance --- in the original \exptwo, a particular exploration scheme is enforced, while in our case, we average multiple trajectories as done in \pref{line:estimator} making use of the uniformly excited feature assumption (to make sure that $\|w_k\|$ is not too large, we also set it to $\bm{0}$ if $\lambda_{\min}(M_k)$ is too small, where $\lambda_{\min}$ denotes the minimum eigenvalue).
Finally, with these estimators, the policy for epoch $k$ is computed by a standard exponential weight update rule (see \pref{line:exp}).


We emphasize that \algExp does not actually need to maintain an instance of \exptwo for each state, 
but instead only needs to maintain the estimators $w_k$ and calculate $\pi(\cdot|x_t)$ on the fly for each $x_t$, which is even more efficient than optimism-based algorithms.
It also enjoys a favorable regret guarantee of order $\otil(\sqrt{T})$, as shown below.
Once again, the best existing result under the same set of assumptions is $\otil(T^{2/3})$ from~\citep{hao2020provably}.
\begin{theorem}\label{thm: MDPEXP2 bound}
     Under Assumptions~\ref{ass: uniform mixing},~\ref{ass: linear policy value},~\ref{ass: assump A4}, \algExp ensures $\E[\Reg_T] = \otil\left(\frac{1}{\sigma}\sqrt{\tmix^3 T}  \right)$.
\end{theorem}

Note that while the bound in \pref{thm: MDPEXP2 bound} seemingly does not depend on $d$, the dependence is in fact implicit because $\frac{1}{\sigma}= \Omega(d)$ always holds by the definition of $\sigma$ (see \pref{rem:sigma} in the appendix). We provide a proof for this fact along with the proof of \pref{thm: MDPEXP2 bound} in the appendix. 

\paragraph{Unknown $\tmix$ and $\sigma$.}
To decide the epoch length and the trajectory length, \algExp requires the prior knowledge of $\tmix$ and $\sigma$. However, if such knowledge is not available, we can still get a slightly worsened asymptotic regret bound of $\otil(T^{1/2+\xi})$, with an additional constant regret of $C^{1/\xi}$ for some constant $C$ that is related to $\tmix$ and $\sigma$. The idea is to slowly increase epoch length and trajectory length with time, and make sure that they exceed the required amount in the long run.  The details are provided in \pref{app: unknown param}. 

\paragraph{Comparison to POLITEX and AAPI.}
Our algorithm \algExp is closely related to the POLITEX algorithm of~\citep{abbasi2019politex} and its improved version AAPI~\citep{hao2020provably}.
The key difference is the way we construct the estimator $w_k$, which at a high level provides a better bias-variance trade-off.
More concretely, our construction is almost unbiased (see \pref{lemma: wk and theta}), but with a larger variance, while the construction for POLITEX/AAPI has a larger bias.
Because of this large bias, POLITEX/AAPI uses a much longer epoch to ensure that the error of $w_k$ is small (i.e., using $\Theta(1/\epsilon^2)$ samples to construct $w_k$ to ensure that $w_k$'s error is $O(\epsilon)$); this results in less frequent update of the policy. In contrasts, MDP-EXP2 only uses a constant (in terms of $t_{mix}$ and $\sigma$) samples to construct $w_k$, and update policies more often. In short, although the $w_k$ of MDP-EXP2 is noisier,  it allows faster updates of the policies and the 
noise of $w_k$ is amortized over epochs. This finally leads to a better regret bound compared to POLITEX/AAPI. 

\paragraph{Connections to Natural Policy Gradient.}
Finally, we remark that although \algExp is based on an linear bandit algorithm \exptwo, it is related to the (in fact much earlier) reinforcement learning algorithm Natural Policy Gradient (NPG)~\cite{kakade2002natural} under softmax parameterization. The connection between softmax-parameterized NPG and the exponential weight update was formalized in a recent work by \citet{agarwal2019optimality}. In \pref{app: NPG}, we first restate the connection. Then we compare the implementation details of \algExp and the NPG algorithm in \cite{agarwal2019optimality}, showing that \algExp improves the sample complexity bound of \cite{agarwal2019optimality} under the considered setting.

%% file: appendix-optimisticLSVI.tex

\section{Auxiliary Lemmas Related to \pref{assump: Linear MDP} and \pref{ass: linear policy value}}

In this section, we provide justification for the scaling assumption made in \pref{assump: Linear MDP} and \pref{ass: linear policy value}, showing that they are indeed without loss of generality as long as one transforms and normalizes the features in some way beforehand.

\label{app: MVEE section}
\begin{lemma}\label{lemma: mvee lemma}
    Let $\Phi=\{\Phi(x,a): x\in\calX, a\in\calA\}\subset \mathbb{R}^d$ be a feature set with rank $d$. Then there exists an invertible linear transformation $v\rightarrow Av$ with $A\in \mathbb{R}^{d\times d}$ such that for any function $F: \calX\times \calA\rightarrow \mathbb{R}$ defined by 
    \begin{align*}
        F(x,a) = \Phi(x,a)^\top z, 
    \end{align*} 
    for some $z \in \mathbb{R}^d$, we have $\|A\Phi(x,a)\|\leq 1$ and $\|A^{-1}z\|\leq \sqrt{d}F_{\max}$ where $F_{\max}\triangleq \sup_{x,a}|F(x,a)|$. 
     
\end{lemma}

    This lemma implies that if we use the transformed feature $\Phi'(x,a)=A\Phi(x,a)$ with $\|\Phi'(x,a)\|\leq 1$, then any function $F(x,a)=\Phi(x,a)^\top z$ can be equivalently written as $F(x,a)=\Phi'(x,a)^\top z'$ with $z' = A^{-1}z$ and $\|z'\|\leq \sqrt{d}F_{\max}$. 
    Therefore, taking $z$ to be $\MU(\calX)$ or $\THE$ for \pref{assump: Linear MDP}, or $w^{\pi}$ for \pref{ass: linear policy value}, with the corresponding $F(x,a)$ being $\int_{\calX}p(x'|x,a)\de x'$, $r(x,a)$, and $q^{\pi}(x,a)$,
    and $F_{\max}$ being $1$, $1$, and $6\tmix$ (\pref{lem:uniformly mixing lemma}) respectively, justifies the scaling stated in these assumptions. 
    
    Notice that the transformation $A$ only depends on the feature set $\Phi$, but not $F$ or $z$. Thus we can perform this transformation as long as we know the feature map. This is similar to a standard preprocessing step of feature normalizing in machine learning. 
    

\begin{proof}[Proof of \pref{lemma: mvee lemma}]
   Define $-\Phi=\{-\Phi(x,a): x\in\calX, a\in\calA\}$ and 
$
        \calK(\Phi)=\Phi\cup -\Phi. 
$
    We first argue that for any bounded feature set $\Phi\subset\mathbb{R}^d$, there exists an invertible linear transformation $v\rightarrow Av$ with $A\in\mathbb{R}^{d\times d}$ such that the \emph{minimum volume enclosing ellipsoid} (MVEE) of the transformed feature set $\calK(A\Phi)$ where $A\Phi\triangleq \{A\Phi(x,a): x\in\calX, a\in\calA\}$ is the unit sphere. This can be seen by the following:
    notice that $\calK(\Phi)$ is always symmetric around the origin, and so is its MVEE. Suppose that the MVEE of $\calK(\Phi)$ is $\{u\in\mathbb{R}^d: u^\top Bu=1\}$ for some invertible $B$ (otherwise $\Phi$ is not full-rank). Then if we pick $A=B^{\frac{1}{2}}$, the MVEE of $\calK(A\Phi)$ will be the unit sphere. 
    
    Now consider this new feature $\Phi'(x,a)\triangleq A\Phi(x,a)$ with the MVEE of $\calK(\Phi')$ being the unit sphere (which implies $\|\Phi'(x,a)\|\leq 1$). Defining $z'=A^{-1}z$, we have $\Phi'(x,a)^\top z' = \Phi(x,a)^\top z=F(x,a)$. Below, we show that $\|z'\|\leq \sqrt{d}F_{\max}$. 
    
    By \pref{lemma: ball theorem} below, there exists a subset $\calM=\{u_1, \ldots, u_m\}\subseteq \calK(\Phi')$ that lie on the unit sphere, and non-negative weights $c_1, \ldots, c_m$, such that
    \begin{align*}
        \sum_{i=1}^m c_i u_i u_i^\top = I_{d}. 
    \end{align*}
    Taking trace on both sides, we get $\sum_{i=1}^{m} c_i=d$. 
    
    Note that we have $F(x,a) = \Phi'(x,a)^\top z'$ for all $x,a$. Specially, applying this to the elements in $\calM$, and using the fact that $|F(x,a)|\leq F_{\max}$, we get  
    \begin{align*}
        dF_{\max}^2 = \sum_{i=1}^m c_i F_{\max}^2 \geq  \sum_{i=1}^m c_i (u_i^\top z')^2 = z^{\prime\top} \left(\sum_{i=1}^m c_i u_iu_i^\top \right) z' = \|z'\|^2, 
    \end{align*}
    which implies $\|z'\|\leq \sqrt{d}F_{\max}$ and finishes the proof.
\end{proof}    

\begin{lemma}
    \label{lemma: ball theorem}
    \textup{(\cite[Theorem 6]{hazan2016volumetric}, \cite{ball1997elementary})}
    Let $\calK$ be a symmetric set such that its MVEE is the unit sphere. Then there exist $m\leq d(d+1)/2-1$ contact points of $\calK$ and the sphere $u_1, \ldots u_m$ and non-negative weights $c_1, \ldots, c_m$ such that $\sum_i c_i u_i=0$ and $\sum_i c_iu_iu_i^\top = I_d$. 
\end{lemma}

\section{Auxiliary Lemmas for Self-normalized Processes}
In this section, we provide some useful lemmas related to the concentration of \emph{self-normalized processes}. The first two are taken directly from~\citep[Appendix D.2]{jin2019provably}. 

\begin{lemma}[Concentration of Self-Normalized Processes]
     \label{lem:self-normalized process}
     Let $\{\varepsilon_t\}_{t=1}^\infty$ be a real-valued stochastic process with corresponding filtration $\{\calF_t\}_{t=0}^{\infty}$. Let $\varepsilon_t|\calF_{t-1}$ be zero-mean and $\sigma$-subgaussian, that is, $\E[\varepsilon_t|\calF_{t-1}] = 0$ and $\E[e^{\lambda\varepsilon_t}|\calF_{t-1}] \leq e^{\lambda^2\sigma^2/2}$ for all $\lambda \in \mathbb{R}$.
     
     Let $\{\PHI_t\}_{t=0}^\infty$ be an $\mathbb{R}^d$-valued stochastic process where $\PHI_{t}\in\calF_{t-1}$. Assume that $\Lambda_1$ is a $d\times d$ positive definite matrix, and let $\Lambda_t=\Lambda_1+\sum_{s=1}^{t-1} \PHI_s\PHI_s^\top $. Then for any $\delta>0$, with probability at least $1-\delta$, we have for all $t>0$, 
     \begin{align*}
          \left\|  \sum_{s=1}^{t-1}\PHI_s\varepsilon_s \right\|_{\Lambda_t^{-1}}^2 \leq 2\sigma^2 \log\left[\frac{\det(\Lambda_t)^{1/2}\det(\Lambda_1)^{-1/2}}{\delta}\right]. 
     \end{align*}
\end{lemma}

\begin{lemma}\label{lem:concentration function class}
    Let $\{x_t\}_{t=1}^\infty$ be a stochastic process on state space $\calX$ with corresponding filtration $\{\calF_{t}\}_{t=0}^\infty$, $\{\PHI_{t}\}_{t=0}^\infty$ be an $\mathbb{R}^d$-valued stochastic process where $\PHI_t\in\calF_{t-1}$ and $\|\PHI_t\|\leq 1$, $\Lambda_{t}=\lambda I + \sum_{s=1}^{t-1}\PHI_s \PHI_s^\top $, and $\calV \subseteq \mathbb{R}^\calX$ be an arbitrary set of functions defined on $\calX$, with $\calN_\varepsilon$ being its $\varepsilon$-covering number with respect to $\text{dist}(v, v')=\sup_x |v(x)-v(x')|$ for some fixed $\varepsilon > 0$.    
    Then for any $\delta>0$, with probability at least $1-\delta$, for all $t>0$ and any $v\in\calV$ so that $\sup_{x}|v(x)|\leq H$, we have 
\begin{align*}
     \left\|   \sum_{s=1}^{t-1}  \PHI_s \Big(  v(x_s) - \E[v(x_s)|\calF_{t-1}]  \Big) \right\|_{\Lambda_t^{-1}}^2 \leq 4H^2\left[  \frac{d}{2}\log\left(\frac{t+\lambda}{\lambda}\right) + \log \frac{\calN_\varepsilon}{\delta} \right] + \frac{8t^2 \varepsilon^2}{\lambda}.
\end{align*}
\end{lemma}

\begin{lemma}\label{lem:covering number}
     Let $\calV$ be a class of mappings from $\calX$ to $\mathbb{R}$ parametrized by $\alpha=(\alpha_1, \alpha_2, \ldots, \alpha_P)\in\mathbb{R}^{P}$ with $\alpha_i\in[-B,B]$ for all $i$.  Suppose that for any $v \in \calV$ (parameterized by $\alpha$) and $v' \in \calV$ (parameterized by $\alpha'$), the following holds: 
    \begin{align*}
         \sup_{x\in\calX}|v(x) - v'(x)| \leq L\|\alpha-\alpha'\|_1. 
     \end{align*}
     Let $\calN_\varepsilon$ be be the $\varepsilon$-covering number of $\calV$ with respect to the distance $\text{dist}(v,v')=\sup_{x\in\calX}|v(x)-v(x')|$. Then 
     \begin{align*}
         \log \calN_{\varepsilon} \leq P\log \left(\frac{2BLP}{\varepsilon}\right). 
     \end{align*}
     
\end{lemma}
\begin{proof}    
      If $\alpha$ and $\alpha'$ are such that $|\alpha_i-\alpha_i'|\leq \frac{\varepsilon}{LP}$ for all $i$, then we have 
     \begin{align*}
         \text{dist}(v,v')=\sup_{x\in\calX} |v(x)-v'(x)| \leq L\times \sum_{i=1}^P |\alpha_i-\alpha_i'| \leq \varepsilon. 
     \end{align*}
     Therefore, the following set constitutes an $\varepsilon$-cover for $\calV$: 
     \begin{align*}
          \left\{ \alpha\in\mathbb{R}^P: \alpha_i = \frac{k\varepsilon}{LP} \text{ for some } k\in \mathbb{Z} \right\} \cap [-B,B]^{P} 
     \end{align*}
     The number of elements in this sets is upper bounded by $\left(\frac{2BLP}{\varepsilon}\right)^P$. 
\end{proof}

\section{Omitted Analysis in \pref{sec:optimism}}

\begin{proof}[Proof of \pref{lem: w* existence}]
     By the two assumptions, we have (with $\e_1 = (1, 0, \ldots, 0)$)
     \begin{align*}
         q^*(x,a) &= r(x,a) - J^* + \E_{x'\sim p(\cdot|x,a)}[v^*(x')] \\
         &= \Phi(x,a)^\top \THE - J^*\Phi(x,a)^\top \e_1 + \Phi(x,a)^\top \int_{\calX} v^*(x') \de \MU(x') \\
         &= \Phi(x,a)^\top \left(\THE - J^*\e_1 + \int_{\calX} v^*(x') \de \MU(x')\right). 
     \end{align*}
     Therefore, we can define $w^* = \THE - J^*\e_1 + \int_{\calX} v^*(x') \de \MU(x')$, proving the first claim. Furthermore, 
     \begin{align*}
          \|w^*\|\leq \|\THE\| + 1 + \sup_{x'\in\calX}|v^*(x')| \times \|\MU(\calX)\| \leq \sqrt{d} + 1 +\tfrac{1}{2}\spv\times \sqrt{d} \leq (2 + \spv)\sqrt{d},
     \end{align*}
     which proves the second claim.
\end{proof}

\subsection{Omitted Analysis in \pref{sec: optimistic q-learning with function approximation}}

\begin{proof}[Proof of \pref{lemma: J optimistic}]
     It suffices to show that with probability at least $1-\delta$, $(w^*, b, J^*)$ for some $b$ is a feasible solution of the optimization problem (since $J_t$ is the optimal solution). 
     To show this, first note that 
     \begin{align*}
          w^* 
          &= \Lambda_t^{-1}\sum_{\tau=1}^{t-1}\Phi(x_\tau, a_\tau)\Phi(x_\tau,a_\tau)w^* + \lambda\Lambda_t^{-1}w^* \tag{definition of $\Lambda_t$}\\
          &= \Lambda_t^{-1}\sum_{\tau=1}^{t-1}\Phi(x_\tau, a_\tau)\left(r(x_{\tau},a_{\tau}) - J^* + \E_{x'\sim p(\cdot|x_\tau,a_\tau)} v^*(x') \right) + \lambda\Lambda_t^{-1}w^* \tag{$q^*(x_\tau, a_\tau) = \Phi(x_\tau,a_\tau)w^*$ and \pref{eq: bellman}}\\
          &= \Lambda_t^{-1}\sum_{\tau=1}^{t-1}\Phi(x_\tau, a_\tau)\left(r(x_{\tau},a_{\tau}) - J^* + v^*(x_{\tau+1}) \right) + \lambda\Lambda_t^{-1}w^* + \epsilon^*_t, 
     \end{align*}
     where 
     \begin{align*}
         \epsilon_t^* = \Lambda_t^{-1}\sum_{\tau=1}^{t-1}\Phi(x_\tau, a_\tau)\left(\E_{x'\sim p(\cdot|x_\tau,a_\tau)} v^*(x')- v^*(x_{\tau+1})\right).
     \end{align*}
     Using \pref{lem:self-normalized process} with $\varepsilon_\tau = \E_{x'\sim p(\cdot|x_\tau,a_\tau)} v^*(x')- v^*(x_{\tau+1})$ and $\PHI_\tau = \Phi(x_\tau, a_\tau)$, we have with probability at least $1-\delta$ (note that given the past $\varepsilon_\tau$ is zero-mean and in the range $[-\spv, \spv]$ thus $\spv$-subgaussian),
     \begin{align*}
         \|\epsilon_t^*\|_{\Lambda_t} &= \left\|\sum_{\tau=1}^{t-1} \PHI_\tau \varepsilon_\tau \right\|_{\Lambda_t^{-1}} \leq \sqrt{2}\spv \sqrt{\log \frac{\det(\Lambda_t)^{1/2} / \det (\Lambda_1)^{1/2}}{\delta} } \\
         &\leq \sqrt{2}\spv \sqrt{\log \frac{(1+\frac{2T}{\lambda d})^{d/2}}{\delta}}  \leq \frac{\beta}{2},
     \end{align*}
     where we use the fact 
     \[
     \det(\Lambda_t) \leq \left(\frac{\trace{\Lambda_t}}{d}\right)^d
     = \left(\frac{\lambda d + \sum_{\tau=1}^{t-1}\|\PHI_\tau\|^2}{d}\right)^d
     \leq \left(\frac{\lambda d + 2T}{d}\right)^d
     \]
     and the definition of $\beta$.  
     Also, $\lambda\|\Lambda_t^{-1}w^*\|_{\Lambda_t} = \lambda\|w^*\|_{\Lambda_t^{-1}} \leq \sqrt{\lambda}\|w^*\| \leq (2 + \spv)\sqrt{\lambda d} \leq \frac{\beta}{2}$ (\pref{lem: w* existence}).
     Define $b = \lambda\Lambda_t^{-1}w^*+\epsilon_t^*$,
     we have thus proven that $\|b\|_{\Lambda_t} \leq \beta$ holds with probability at least $1-\delta$.
     which proves that $(w^*, b, J^*)$ is a solution of the optimization problem, finishing the proof.
\end{proof}

\begin{proof}[Proof of \pref{thm: optimistic AR}]
     Without loss of generality, we assume $\spv\leq \sqrt{T}, d\leq \sqrt{T}$, and $T\geq 16$ (otherwise the bound is vacuous). 
     Fix $t$ and let $s=s_t$. Define 
     \begin{align*}
          \epsilon_s = \Lambda_s^{-1}\sum_{\tau=1}^{s-1}\Phi(x_\tau,a_\tau)\left(v_s(x_{\tau+1}) - \E_{x'\sim p(\cdot|x_\tau,a_\tau)}v_s(x')\right).
     \end{align*}
     Using the identity
     \begin{align*}
     w^* &= \Lambda_{s}^{-1} \sum_{\tau=1}^{s-1}\Phi(x_\tau, a_\tau)\Phi(x_\tau, a_\tau)^\top w^* + \lambda\Lambda_s^{-1}w^* \\
     &= \Lambda_{s}^{-1} \sum_{\tau=1}^{s-1}\Phi(x_\tau, a_\tau)\left(r(x_\tau, a_\tau) - J^* + \E_{x'\sim p(\cdot|x_\tau,a_\tau)}v^*(x')\right) + \lambda\Lambda_s^{-1}w^*,
    \end{align*}
     and the definition of $w_s$, we have
     \begin{align*}
          &w_s - w^*\\
          &= \Lambda_{s}^{-1} \sum_{\tau=1}^{s-1}\Phi(x_\tau, a_\tau)\left(r(x_\tau, a_\tau) - J_s + v_s(x_{\tau+1}) \right) + b_s \\
          &\qquad\qquad - \Lambda_{s}^{-1} \sum_{\tau=1}^{s-1}\Phi(x_\tau, a_\tau)\left(r(x_\tau, a_\tau) - J^* + \E_{x'\sim p(\cdot|x_\tau,a_\tau)}v^*(x')\right) - \lambda\Lambda_s^{-1}w^*  \\
          &= \Lambda_{s}^{-1} \sum_{\tau=1}^{s-1}\Phi(x_\tau, a_\tau)\left(J^* - J_s + \E_{x'\sim p(\cdot|x_\tau,a_\tau)} [v_s(x')-v^*(x')] \right)  +\epsilon_s  + b_s - \lambda\Lambda_s^{-1}w^* \\
          &= \Lambda_{s}^{-1} \sum_{\tau=1}^{s-1}\Phi(x_\tau, a_\tau)\Phi(x_\tau,a_\tau)^\top\left(J^*\e_1 - J_s\e_1 + \int_\calX\left(v_s(x')-v^*(x')\right)\de\MU(x') \right)  +\epsilon_s + b_s - \lambda\Lambda_s^{-1}w^*  \\
          &=J^*\e_1 - J_s\e_1 + \int_\calX\left(v_s(x')-v^*(x')\right)\de\MU(x') +\epsilon_s + b_s \\
          &\qquad \qquad -\lambda \Lambda_s^{-1} \left(J^*\e_1 - J_s\e_1 + \int_\calX\left(v_s(x')-v^*(x')\right)\de\MU(x') \right)  - \lambda\Lambda_s^{-1}w^*.
     \end{align*}

     Therefore, 
     \begin{align}
          &q_{s}(x_t,a_t) - q^*(x_t,a_t)   =\Phi(x_t,a_t)^\top (w_{s}-w^*)    \nonumber \\
          &\leq (J^*-J_s) + \E_{x'\sim p(\cdot|x_t,a_t)}[v_s(x')-v^*(x')] + \Phi(x_t,a_t)^\top (\epsilon_s + b_s + \lambda \Lambda_s^{-1} u_s),   \nonumber
     \end{align}
     where $u_s\triangleq -\left(J^*\e_1 - J_s\e_1 + \int_\calX\left(v_s(x')-v^*(x')\right)\de\MU(x') \right) - w^*$. 
     
     Next, under the event $J^* \leq J_s$ which holds with probability at least $1-\delta$ (\pref{lemma: J optimistic}), we continue with
     \begin{align}
      &q_{s}(x_t,a_t) - q^*(x_t,a_t) \\
      &\leq \E_{x'\sim p(\cdot|x_t,a_t)}[v_s(x')-v^*(x')] + \Phi(x_t,a_t)^\top (\epsilon_s + b_s + \lambda \Lambda_s^{-1} u_s)  \notag \\
          &\leq \E_{x'\sim p(\cdot|x_t,a_t)}[v_s(x')-v^*(x')] + \|\Phi(x_t,a_t)\|_{\Lambda_s^{-1}} \|\epsilon_s + b_s + \lambda \Lambda_s^{-1} u_s\|_{\Lambda_s}   \nonumber  \\
          &\leq \E_{x'\sim p(\cdot|x_t,a_t)}[v_s(x')-v^*(x')] + 2\|\Phi(x_t,a_t)\|_{\Lambda_t^{-1}} \|\epsilon_s + b_s+ \lambda \Lambda_s^{-1} u_s\|_{\Lambda_s},  \label{eq: temporary bound 3}
     \end{align}
     where the second inequality uses \Holder and the last one uses the fact $\Lambda_s \preceq \Lambda_t \preceq 2\Lambda_s$ according to the lazy update schedule of the algorithm.
     
     By the algorithm, $\|b_s\|_{\Lambda_s}\leq \beta$. To bound $\|\epsilon_s\|_{\Lambda_s}$, we use \pref{lem:concentration function class} and \pref{lem:covering number}: Define $\varepsilon_\tau = v_s(x_{\tau+1}) - \E_{x'\sim p(\cdot|x_\tau, a_\tau)}v_s(x')$ and $\PHI_\tau = \frac{1}{\sqrt{2}}\Phi(x_\tau, a_\tau)$. With \pref{lem:concentration function class} and the fact $|v_s(x)|\leq \sqrt{2}\|w_s\| \leq (2+\spv)\sqrt{2d}$, we have that with probability at least $1-\delta$, for all $s$:
     \begin{align*}
         \|\epsilon_s\|_{\Lambda_s} = \sqrt{2}\left\| \sum_{\tau=1}^{s-1} \PHI_\tau \varepsilon_\tau \right\|_{\Lambda_s^{-1}} \leq 
         4(2+\spv)\sqrt{d}\sqrt{\frac{d}{2} \log \frac{s+\lambda}{\lambda} + \log \frac{\calN_\varepsilon}{\delta}} + 4\sqrt{\frac{s^2 \varepsilon^2}{\lambda}}, 
     \end{align*}
     where $\varepsilon=\frac{1}{T}$ and $\calN_\varepsilon$ is the $\varepsilon$-cover for the function class of $v_s$, which can be bounded with the help of \pref{lem:covering number} (with $\alpha=w_s$, $P=d$, $B=(2+\spv)\sqrt{d}$,  and $L=\sqrt{2}$) by
     \begin{align*}
          \log \calN_\varepsilon \leq d\log \frac{2(2+\spv)\sqrt{d}\times \sqrt{2} d}{T^{-2}} \leq 7 d\log T
     \end{align*}
     (using the conditions stated at the beginning of the proof).
     Therefore, we have
     \begin{align}
           \|\epsilon_s\|_{\Lambda_s} \leq 4(2+\spv)\sqrt{d}\sqrt{8d\log T + \log (1/\delta)} + 4 = \order(\beta), \label{eq: tmp bound 1}
     \end{align}
     for all $s$ with probability at least $1-\delta$. 
     Next, we bound $\|\lambda \Lambda_s^{-1}u_s\|_{\Lambda_s}$ as: 
     \begin{align}
          \|\lambda \Lambda_s^{-1}u_s\|_{\Lambda_s} = \lambda \| u_s\|_{\Lambda_s^{-1}}\leq \sqrt{\lambda}\|u_s\| \leq  \order\left(1+(2+\spv)d\right) = \order(\beta),  \label{eq: tmp bound 4}
     \end{align}
     where in the second inequality we use the condition $\|\MU(\calX)\|\leq \sqrt{d}$ in \pref{assump: Linear MDP} to bound $\|\int_\calX\left(v_s(x')-v^*(x')\right)\de\MU(x')\|$ as $\sup_{x\in\calX} |v_s(x)-v^*(x)| \|\MU(\calX)\| = \order((2+\spv)d)$.
     Put together, the above shows $\|\epsilon_s + b_s+ \lambda \Lambda_s^{-1} u_s\|_{\Lambda_s} = \order(\beta)$.
     
     Continuing with \pref{eq: temporary bound 3} and summing over $t$, we have that with probability at least $1-2\delta$,
     \begin{align*}
          &\sum_{t=1}^T  (q_{s_t}(x_t,a_t) - q^*(x_t,a_t))
          \leq  \sum_{t=1}^T \E_{x'\sim p(\cdot|x_t,a_t)}\left[v_{s_t}(x') - v^*(x')\right] +\order\left(\beta \sum_{t=1}^T\|\Phi(x_t,a_t)\|_{\Lambda_t^{-1}} \right) \\
          &= \sum_{t=1}^T \E_{x'\sim p(\cdot|x_t,a_t)}\left[v_{s_t}(x') - v^*(x')\right] +\order\left(\beta \sqrt{T}\sqrt{\sum_{t=1}^T \|\Phi(x_t,a_t)\|_{\Lambda_{t}^{-1} }^2}\right) \tag{Cauchy-Schwarz inequality}\\
          &= \sum_{t=1}^T \E_{x'\sim p(\cdot|x_t,a_t)}\left[v_{s_t}(x') - v^*(x')\right] +\order\left(\beta \sqrt{dT\log T}\right),
     \end{align*}
     where the last equality is by \cite[Lemma D.2]{jin2019provably} with the facts that $\det(\Lambda_1)=\lambda^d$ and $\det(\Lambda_{T+1})\leq \left(\frac{1}{d}\text{trace}(\Lambda_{T+1})\right)^d\leq (\lambda + 2T)^d$.     
     Rearranging the last inequality we get
     \begin{align}
          &\sum_{t=1}^T \left(\E_{x'\sim p(\cdot|x_t,a_t)}[v^*(x')] - q^*(x_t,a_t)\right)   \nonumber \\ 
          &\leq\sum_{t=1}^T \left(\E_{x'\sim p(\cdot|x_t,a_t)}[v_{s_t}(x')]- q_{s_t}(x_t,a_t)\right) +\order\left(\beta \sqrt{dT\log T}\right)\nonumber \\ 
          &= \sum_{t=1}^T \left(\E_{x'\sim p(\cdot|x_t,a_t)}[v_{s_t}(x')]- v_{s_t}(x_t)\right) +\order\left(\beta \sqrt{dT\log T}\right)\nonumber 
     \end{align}    
     where the last line is by the choice of $a_t$.
     Next, notice that every time the algorithm updates (i.e. $s_t\neq s_{t-1}$), it holds that $\det(\Lambda_t)=\det(\Lambda_{s_t})\geq 2\det(\Lambda_{s_{t-1}})$. Since $\det(\Lambda_{T+1})/\det(\Lambda_1)\leq \left(\frac{\lambda + 2T}{\lambda}\right)^d$, this cannot happen more than $\log_2 \left(\frac{\lambda + 2T}{\lambda}\right)^d = \order\left(d\log T\right)$ times.
     Using this fact and the range of $v_t$, we continue with
     \begin{align}       
     &\sum_{t=1}^T \left(\E_{x'\sim p(\cdot|x_t,a_t)}[v^*(x')] - q^*(x_t,a_t)\right)   \nonumber \\  
          &\leq \sum_{t=1}^T \left(\E_{x'\sim p(\cdot|x_t,a_t)}[v_{s_{t+1}}(x')]- v_{s_t}(x_t)\right) + \order\left(\beta \sqrt{dT\log T} + \beta d \log T\right)\nonumber \\
          &=\sum_{t=1}^T \left(\E_{x'\sim p(\cdot|x_t,a_t)}[v_{s_{t+1}}(x')]- v_{s_{t+1}}(x_{t+1})\right) + \order\left(\beta \sqrt{dT\log T} + \beta d \log T\right) \nonumber \\
          &=\order\left(\beta \sqrt{dT\log T} + \beta d \log T\right), \label{eq: tmp bound 7}
     \end{align}
     where the last step holds with probability at least $1-\delta$ by Azuma's inequality.
     Finally, note that the regret can be written as
     \begin{align*}
          \Reg_T&=\sum_{t=1}^T \left(J^* - r(x_t,a_t)\right)
          = \sum_{t=1}^T \left( \E_{x'\sim p(\cdot|x_t,a_t)}[v^*(x')] - q^*(x_t,a_t) \right)\\
          &= \order\left(\beta \sqrt{dT\log T} + \beta d \log T\right).           
     \end{align*}
     by the Bellman optimality equation, which finishes the proof (combining all the high probability statements with a union bound, the last bound holds with probability at least $1-3\delta$). 
\end{proof}

\subsection{Omitted Analysis in \pref{sec:FH}}

 \begin{proof}[Proof of \pref{lem: optimistic Q}]

     By \pref{assump: Linear MDP} and the Bellman equation for the finite-horizon problem (\pref{eq:discounted_bellman}), we have that for any finite-horizon policy $\pi$ and any $h<H$,
     \begin{align*}
          Q_h^{\pi}(x,a) &= r(x,a) +   \E_{x'\sim p(\cdot|x,a)}\left[V_{h+1}^{\pi}(x')\right] \\
          &= \Phi(x,a)^\top \THE +  \Phi(x,a)^\top \int_{\calX}  V_{h+1}^{\pi}(x') \de\MU(x') \\
          &= \Phi(x,a)^\top \left( \THE +   \int_{\calX}  V_{h+1}^{\pi}(x') \de \MU(x')\right). 
     \end{align*}
     Define $w^{\pi}_h=\THE + \int_{\calX}  V_{h+1}^{\pi}(x') \de \MU(x')$. Then we have $Q_h^{\pi}(x,a)=\Phi(x,a)^\top w^{\pi}_h$ with $\|w_h^{\pi}\|\leq \|\THE\| + (H-h)\|\MU(\calX)\|\leq \sqrt{d} + \sqrt{d}(H-h)\leq \sqrt{d}H$. 
     
     We now rewrite $w_{h}^k - w_h^{\pi}$ as follow. For simplicity, we denote $x\sim p(\cdot|x^{k'}_{h'}, a^{k'}_{h'})$ as $x\sim (k',h')$, $\Phi(\xaprime)$ as $\Phiprime$, and $r(\xaprime)$ as $\rprime$
     \begin{align*}
        &w_{h}^k -  w^{\pi}_h \\
        &= \Lambda_{k}^{-1}\sum_{k'=1}^{k-1}\sum_{h'=1}^{H}\Phiprime\left[\rprime + V_{h+1}^{k}(x_{h'+1}^{k'}) \right] - \Lambda_k^{-1}\left( \lambda I + \sum_{k'=1}^{k-1}\sum_{h'=1}^{H}\Phiprime{\Phiprime}^\top \right)w_h^{\pi}  \\
        &= \Lambda_{k}^{-1}\sum_{k'=1}^{k-1}\sum_{h'=1}^{H}\Phiprime\left[\rprime +  V_{h+1}^k(x_{h'+1}^{k'}) \right] \\
        &\qquad \qquad - \Lambda_k^{-1}\sum_{k'=1}^{k-1}\sum_{h'=1}^{H} \Phiprime\left[\rprime +  \E_{x'\sim (k',h')}[V_{h+1}^{\pi}(x')]\right] - \lambda \Lambda_k^{-1}w^{\pi}_h \tag{using $Q_h^{\pi}(x,a)=\Phi(x,a)^\top w^{\pi}_h$ and the Bellman equation}\\
        &= \Lambda_{k}^{-1}\sum_{k'=1}^{k-1}\sum_{h'=1}^{H} \Phiprime\left[V_{h+1}^k(x_{h'+1}^{k'}) -  \E_{x'\sim (k',h')}V_{h+1}^{\pi}(x') \right] - \lambda\Lambda_{k}^{-1}w^{\pi}_h \\
        &= \Lambda_{k}^{-1}\sum_{k'=1}^{k-1}\sum_{h'=1}^{H}\Phiprime\left[\E_{x'\sim (k',h')} V_{h+1}^k(x') - \E_{x'\sim (k',h')}V_{h+1}^{\pi}(x') \right] + \epsilon_{h}^k - \lambda\Lambda_{k}^{-1}w^{\pi}_h \tag{define $\epsilon_{h}^k = \Lambda_{k}^{-1}\sum_{k'=1}^{k-1}\sum_{h'=1}^{H}\Phiprime\left[ V_{h+1}^k(x_{h'+1}^{k'}) -\E_{x'\sim (k',h')}\left[V_{h+1}^k(x')\right]\right]$}  \\
        &=  \Lambda_{k}^{-1}\sum_{k'=1}^{k-1}\sum_{h'=1}^{H} \Phiprime {\Phiprime}^\top\left[ \int_\calX  (V_{h+1}^k(x') -  V_{h+1}^{\pi}(x')) \de\MU(x')\right] + \epsilon_{h}^k- \lambda\Lambda_{k}^{-1}w^{\pi}_h  \\
        &= \left(I-\lambda \Lambda_{k}^{-1}\right) \left[ \int_{\calX}  (V_{h+1}^k(x') -  V_{h+1}^{\pi}(x'))\de\MU(x') \right] + \epsilon_{h}^k - \lambda\Lambda_{k}^{-1}w^{\pi}_h  \\
        &= \int_{\calX}  (V_{h+1}^k(x') - V_{h+1}^{\pi}(x'))\de\MU(x') + \epsilon_{h}^k -\lambda\Lambda_{k}^{-1} \left[ \int_{\calX}  (V_{h+1}^k(x') -  V_{h+1}^{\pi}(x'))\de\MU(x') \right] - \lambda\Lambda_{k}^{-1}w^{\pi}_h.
     \end{align*}
     Therefore, 
     \begin{align}
          &\hatQ_{h}^k(x,a) - Q_h^{\pi}(x,a)  \nonumber  \\
          &= \Phi(x,a)^\top (w_{h}^k -w_h^{\pi}) + \beta \sqrt{\Phi(x,a)^\top \Lambda_k^{-1} \Phi(x,a)}  \nonumber  \\
          &= \Phi(x,a)^\top \int_\calX (V_{h+1}^k(x')-V_{h+1}^{\pi}(x'))\de\MU(x')+ \Phi(x,a)^\top \epsilon_{h}^k + \beta \|\Phi(x,a)\|_{\Lambda_k^{-1}}  \nonumber  \\
          &\qquad \qquad \qquad - \lambda  \Phi(x,a)^\top  \Lambda_k^{-1} \left[ \int_{\calX}  (V_{h+1}^k(x') -  V_{h+1}^{\pi}(x'))\de\MU(x') \right]  - \lambda\Phi(x,a)^\top \Lambda_{k}^{-1}w_h^{\pi}  \nonumber  \\
          &= \E_{x'\sim p(\cdot|x,a)} \left[V_{h+1}^k(x')-V_{h+1}^{\pi}(x')\right] + \underbrace{\Phi(x,a)^\top \epsilon_{h}^k}_{\term_1} + \beta \|\Phi(x,a)\|_{\Lambda_k^{-1}}  \nonumber  \\
          &\qquad \qquad \qquad \underbrace{- \lambda  \Phi(x,a)^\top  \Lambda_k^{-1} \left[ \int_{\calX}  (V_{h+1}^k(x') -  V_{h+1}^{\pi}(x'))\de\MU(x') \right]}_{\term_2} \underbrace{ - \lambda\Phi(x,a)^\top \Lambda_{k}^{-1}w_h^{\pi} }_{\term_3}.   \label{eq: tmp optimistic bound}
          \end{align}
          Below we bound the manitudes of $\term_1, \term_2, \term_3$ respectively. For $\term_1$, we use \pref{lem:concentration function class} and \pref{lem:covering number}: define $\varepsilon_{h'}^{k'} = V_{h+1}^k (x_{h'+1}^{k'}) -\E_{x'\sim (k',h')}\left[V_{h+1}^k(x')\right]$, $\PHI_{h'}^{k'}=\frac{1}{\sqrt{2}}\Phiprime$. By \pref{lem:concentration function class}, we have 
          \begin{align}
               \|\epsilon_h^k\|_{\Lambda_k} 
               &= \sqrt{2} \left\| \Lambda_{k}^{-1} \sum_{k'=1}^{k-1} \sum_{h'=1}^{H} \PHI_{h'}^{k'} \varepsilon_{h'}^{k'}  \right\|_{\Lambda_k}  \nonumber \\
               &= \sqrt{2} \left\|  \sum_{k'=1}^{k-1} \sum_{h'=1}^{H} \PHI_{h'}^{k'} \varepsilon_{h'}^{k'}  \right\|_{\Lambda_k^{-1}}    \nonumber  \\
               &\leq 2\sqrt{2}H \sqrt{\frac{d}{2} \log\frac{T+\lambda}{\lambda} + \log \frac{\calN_\varepsilon}{\delta} } + \sqrt{2}\times\sqrt{\frac{8t^2 \varepsilon^2}{\lambda}},    \label{eq: tmp bound of epsilon}
          \end{align}
          for all $k$ and $h$ with probability at least $1-\delta$, 
          where $\calN_\varepsilon$ is the $\varepsilon$-cover of the function class that $V_{h+1}^k(\cdot)$ lies in. Notice that all $t$, $V_{h+1}^k(\cdot)$ can be expressed as the following: 
          \begin{align*}
                V_{h+1}^k(x) =  \min\left\{\max_a w^\top \Phi(x,a) + \beta \sqrt{\Phi(x,a)^\top \Gamma\Phi(x,a)}, \ \ \ \  H\right\}
          \end{align*} 
          for some positive definite $\Gamma\in\mathbb{R}^{d\times d}$ with $1=\frac{1}{\lambda}\geq \lambda_{\max}(\Gamma) \geq \lambda_{\min}(\Gamma) \geq \frac{1}{\lambda + 2T}=\frac{1}{1+2T}$ and some $w \in\mathbb{R}^d$ with $\|w\|\leq \lambda_{\max}(\Gamma)\times T \times \sup_{x,a,x'} \left(\|\Phi(x,a)\|H\right)\leq \sqrt{2}TH$. Therefore, we can write the class of functions that $V_{h+1}^k(\cdot)$ lies in as follows: 
          \begin{align*}
              \calV &= \bigg\{V(x) =  \min\left\{\max_a w^\top \Phi(x,a) + \beta \sqrt{\Phi(x,a)^\top \Gamma\Phi(x,a)}, \ \ \ \ H\right\} : \\
               &\qquad \qquad w\in\mathbb{R}^d:  \|w\|\leq \sqrt{2}TH, \quad \Gamma\in \mathbb{R}^{d\times d}: \frac{1}{1+2T} \leq \lambda_{\min}(\Gamma) \leq \lambda_{\max}(\Gamma) \leq 1 \bigg\}.  
          \end{align*}

Now we apply \pref{lem:covering number} to $\calV$, with the following choices of parameters: $\alpha=(w, \Gamma)$, $P = d^2 + d$, $\varepsilon=\frac{1}{T}$, $B=\sqrt{2}TH$, and $L=\beta\sqrt{2(1+2T)}$ which is given by the following calculation: for any $\Delta w = \epsilon\e_i$, 
\begin{align*}
     \frac{1}{|\epsilon|}\left|(w+\Delta w)^\top \Phi(x,a) - w^\top \Phi(x,a)\right| = |\e_i^\top \Phi(x,a)|\leq \|\Phi(x,a)\|\leq \sqrt{2}, 
\end{align*}
and for any $\Delta \Gamma=\epsilon \e_i\e_j^\top$, 
\begin{align*}
    &\frac{1}{|\epsilon|}\left|\beta \sqrt{\Phi(x,a)^\top (\Gamma+\Delta\Gamma)\Phi(x,a)} - \beta \sqrt{\Phi(x,a)^\top \Gamma\Phi(x,a)}\right| \\
    &\leq \beta \frac{\left|\Phi(x,a)^\top \e_i\e_j^\top \Phi(x,a)\right|}{\sqrt{\Phi(x,a)^\top \Gamma \Phi(x,a)}}   \tag{$\sqrt{u+v}-\sqrt{u}\leq \frac{|v|}{\sqrt{u}}$}\\
    &\leq \beta \frac{\Phi(x,a)^\top\left(\frac{1}{2}\e_i\e_i^\top +\frac{1}{2}\e_j\e_j^\top\right)\Phi(x,a)}{\sqrt{\Phi(x,a)^\top \Gamma \Phi(x,a)}} \\
    &\leq \beta \frac{\Phi(x,a)^\top\Phi(x,a)}{\sqrt{\Phi(x,a)^\top \Gamma \Phi(x,a)}}\leq \sqrt{2}\beta \sqrt{\frac{1}{\lambda_{\min}(\Gamma)}}\leq \beta\sqrt{2(1+2T)}.
\end{align*}
\pref{lem:covering number} then implies:
\begin{align*}
     \log \calN_{\varepsilon}\leq (d^2+d)\log \frac{2\times \sqrt{2}TH \times \beta\sqrt{2(1+2T)} \times (d^2+d)}{T^{-1}} \leq 20d^2\log T,
\end{align*}
where in the last step we use the definition of $\beta$ and also assume without loss of generality that $\spv\leq \sqrt{T}, d\leq \sqrt{T}$, and $T\geq 32$ (since otherwise the regret bound is vacuous). 
     Then by \pref{eq: tmp bound of epsilon} we have with probability $1-\delta$, for all $k$ and $h$,
     \begin{align*}
          \|\epsilon_h^k\|_{\Lambda_k}\leq 2\sqrt{2}H \sqrt{\frac{d}{2} \log \frac{T+1}{1}  + \log \frac{1}{\delta} + 20d^2 \log T} + 4 \leq 20dH\sqrt{\log (T/\delta)}= \frac{\beta}{2},  
     \end{align*}
     and therefore, 
          \begin{align*}
              |\term_1| &\leq \|\Phi(x,a)\|_{\Lambda_k^{-1}}\|\epsilon_h^{k}\|_{\Lambda_k} \leq \frac{\beta}{2}\|\Phi(x,a)\|_{\Lambda_k^{-1}}. 
          \end{align*}
          Furthermore, 
          \begin{align*}
              |\term_2| &\leq \|\Phi(x,a)\|_{\Lambda_k^{-1}}\left\| \lambda \int_{\calX}  (V_{h+1}^{k}(x') -  V_{h+1}^{\pi}(x'))\de\MU(x')  \right\|_{\Lambda_k^{-1}} \tag{Cauchy-Schwarz inequality}\\
              &\leq \|\Phi(x,a)\|_{\Lambda_k^{-1}}\left\| \sqrt{\lambda} \int_{\calX}  (V_{h+1}^k (x') -  V_{h+1}^{\pi}(x'))\de\MU(x')  \right\| \tag{$\lambda_{\min}(\Lambda_k) \geq \lambda$} \\
              &\leq  \sqrt{\lambda}\|\Phi(x,a)\|_{\Lambda_k^{-1}} \times H\sqrt{d} \tag{$\|\MU(\calX)\| \leq \sqrt{d}$ by \pref{assump: Linear MDP}} \\
              &\leq \frac{\beta}{4} \|\Phi(x,a)\|_{\Lambda_k^{-1}},  \tag{using $\lambda=1$}
          \end{align*}    
          and 
          \begin{align*}    
              |\term_3| &\leq \|\Phi(x,a)\|_{\Lambda_k^{-1}}\|\lambda w^{\pi}_h\|_{\Lambda_k^{-1}}  \tag{Cauchy-Schwarz inequality}\\
              &\leq \|\Phi(x,a)\|_{\Lambda_k^{-1}} \left\|\sqrt{\lambda}  w_h^{\pi} \right\|  \tag{$\lambda_{\min}(\Lambda_k) \geq \lambda$} \\
              &\leq \frac{\beta}{4}\|\Phi(x,a)\|_{\Lambda_k^{-1}} \tag{$\|w_h^{\pi}\| \leq \sqrt{d}H$ and $\lambda=1$}. 
          \end{align*}
          Therefore, $|\term_1|+|\term_2| + |\term_3|\leq \beta \|\Phi(x,a)\|_{\Lambda_k^{-1}}$ for all $k$ and $h$ with probability at least $1-\delta$. Then by \pref{eq: tmp optimistic bound}, we have 
           \begin{align*}
               \hatQ_h^k(x,a) - Q_h^{\pi}(x,a) \leq \E_{x'\sim p(\cdot|x,a)} [V_{h+1}^k(x') - V_{h+1}^{\pi}(x')] + 2\beta  \|\Phi(x,a)\|_{\Lambda_k^{-1}},
          \end{align*}
          proving one inequality in the lemma statement (since $Q_h^k(x,a) \leq \hatQ_h^k(x,a)$). To prove the other inequality, 
          note that \pref{eq: tmp optimistic bound} together with $|\term_1|+|\term_2| + |\term_3|\leq \beta \|\Phi(x,a)\|_{\Lambda_k^{-1}}$ also implies
          \begin{align}
               \hatQ_h^k(x,a) - Q_h^{\pi}(x,a) \geq \E_{x'\sim p(\cdot|x,a)} [V_{h+1}^k(x') - V_{h+1}^{\pi}(x')]. \label{eq:induction}
          \end{align}
          Now we fix $k$ and use induction on $h$ to prove $Q_h^k(x,a) \geq Q_h^{\pi}(x,a)$.
          The base case $h=H$ is clear due to \pref{eq:induction} and the facts $V_{H+1}^k(x)=V_{H+1}^{\pi}(x)=0$ and $Q_H^k(x,a)-Q_H^{\pi}(x,a) = \min\{\hatQ_H^k(x,a), H\}-Q_H^{\pi}(x,a)\geq 0$.
          Next assume $Q_{h+1}^k(x,a) \geq Q_{h+1}^{\pi}(x,a)$ for all $x$ and $a$.
          Then $V_{h+1}^k(x) = \max_a Q_{h+1}^k(x,a)\geq \max_a Q_{h+1}^{\pi}(x,a) \geq V_{h+1}^{\pi}(x)$.
          Using \pref{eq:induction} we have  $\hatQ_h^k(x,a) - Q_h^{\pi}(x,a)\geq 0$, which again implies $Q_{h}^k(x,a)=\min\{\hatQ_h^k(x,a), H\} \geq Q_{h}^{\pi}(x,a)$.
          This finishes the induction and proves the other inequality in the lemma statement.
\end{proof}

\begin{proof}[Proof of \pref{thm: discounted main thm}]
Let $\pi_k = (\pi_1^{k}, \ldots, \pi_H^{k})$ be the finite-horizon policy that our algorithm executes for episode $k$, that is, $\pi_{h}^{k}(a|x)=\one[a=\argmax_{a'} Q^k_h(x,a')]$ (breaking ties arbitrarily).
Also let $\bar{\pi}^*$ be the optimal finite-horizon policy with value functions 
$Q_h^{*}(x,a) = \max_{\pi} Q_h^{\pi}(x,a)$ and $V_h^*(x) = \max_{a}Q_h^*(x,a)$. 
     We first decompose the regret as
     \begin{align}
     \Reg_T &= \sum_{t=1}^T  \left(J^* - r(x_t,a_t)\right)  \nonumber  \\
     &= \underbrace{\sum_{k=1}^{T/H} \left(HJ^* - V^*_1(x_1^k)\right)  }_{\term_4}
     + \underbrace{\sum_{k=1}^{T/H} \left(V^*_1(x_1^k) - V^{\pi_k}_1(x_1^k) \right)}_{\term_5} + \underbrace{\sum_{k=1}^{T/H} \left(V_1^{\pi_k}(x_1^k) - \sum_{h=1}^H r(x_h^k, a_h^k)\right)}_{\term_6} 
       \label{eq: discounted regret decompose}
     \end{align}
     In \pref{lemma: connecting lemma} (stated after this proof), we connect the optimal reward of the the infinite-horizon setting and the finite-horizon setting and show that $\term_4 \leq \frac{T\spv}{H}$. 
      
     Notice that conditioned on the history before episode $k$, $V_1^{\pi_k}(x_1^k)$ is the expectation of $\sum_{h=1}^H r(x_h^k, a_h^k)$. Therefore, $\term_6$ is a martingale different sequence, which can be upper bounded by $\order\left(H\sqrt{\frac{T}{H}\log(1/\delta)}\right)=\order\left(\sqrt{HT\log(1/\delta)}\right)$ with probabiltiy at least $1-\delta$ (via Azuma's inequality). 
     
     Finally, we deal with $\term_5$. Below we assume that the high-probability event in \pref{lem: optimistic Q} hold. Then for all $k, h$:
     \begin{align*}
         Q_h^k(\xa) - Q_h^{\pi_k}(\xa)  
         &\leq \E_{x'\sim (k,h)} [V_{h+1}^k(x') - V_{h+1}^{\pi_k}(x')] + 2\beta \|\Phi(x_h^k,a_h^k)\|_{\Lambda_{k}^{-1}} \\
         &=  V_{h+1}^k(x_{h+1}^k) - V_{h+1}^{\pi_k}(x_{h+1}^k) + 2\beta \|\Phi(x_h^k,a_h^k)\|_{\Lambda_{k}^{-1}}  + e_h^k \\
         &=  Q_{h+1}^k(x_{h+1}^k, a_{h+1}^k) - Q_{h+1}^{\pi_k}(x_{h+1}^k, a_{h+1}^k) + 2\beta \|\Phi(x_h^k,a_h^k)\|_{\Lambda_{k}^{-1}}  + e_h^k
     \end{align*}
     where in the first equality we define 
     \begin{align*}
          e_h^k = \E_{x'\sim (k,h)} [V_{h+1}^k(x') - V_{h+1}^{\pi_k}(x')]  -  \left(V_{h+1}^k(x_{h+1}^k) - V_{h+1}^{\pi_k}(x_{h+1}^k)\right), 
     \end{align*}
     which has zero mean,
     and in the second equality we use the facts $V_{h+1}^k(x_{h+1}^k)  =  Q_{h+1}^k(x_{h+1}^k, a_{h+1}^k)$ and $V_{h+1}^{\pi_k}(x_{h+1}^k)=Q_{h+1}^{\pi_k}(x_{h+1}^k, a_{h+1}^k)$.
      Repeating the same argument and using $V_{H+1}^k(\cdot)=V_{H+1}^{\pi_k}(\cdot)=0$, we arrive at
     \begin{align*}
           Q_1^k(x_1^k, a_1^k) - Q_1^{\pi_k}(x_1^k, a_1^k)  \leq \sum_{h=1}^H \left(2\beta \|\Phi(x_h^k,a_h^k)\|_{\Lambda_{k}^{-1}}  + e_h^k\right). 
     \end{align*}
     Further using that $V_1^*(x_1^k)=\max_a Q_1^*(x_1^k, a)\leq \max_a Q_1^k(x_1^k, a)= Q_1^k(x_1^k, a_1^k)$ (the inequality is by \pref{lem: optimistic Q}) and that $V_1^{\pi_k}(x_1^k)=Q_1^{\pi_k}(x_1^k, a_1^k)$, we have shown
     \begin{align*}
         \term_5 &\leq \sum_{k=1}^{T/H}\sum_{h=1}^H  \left(\beta \|\Phi(x_h^k,a_h^k)\|_{\Lambda_{k}^{-1}}  + e_h^k \right). 
     \end{align*}
    The term $\sum_{k=1}^{T/H}\sum_{h=1}^H  e^k_h$ is again the sum of a martingale difference sequence with each term's magnitude bounded by $2H$, and therefore is bounded by $\order\left(H\sqrt{T\log(1/\delta)}\right)$ with probability at least $1-\delta$ using Azuma's inequality.
     For the term $\sum_{k=1}^{T/H}\sum_{h=1}^H \beta \|\Phi(x_h^k,a_h^k)\|_{\Lambda_{k}^{-1}}$, we first decompose it into two parts:
     \begin{align*}
         \sum_{k: \det(\Lambda_{k+1})\leq 2\det(\Lambda_k)} \sum_{h=1}^H  \beta\|\Phi(x_h^k,a_h^k)\|_{\Lambda_{k}^{-1}} +  \sum_{k: \det(\Lambda_{k+1}) > 2\det(\Lambda_k)} \sum_{h=1}^H \beta\|\Phi(x_h^k,a_h^k)\|_{\Lambda_{k}^{-1}}. 
     \end{align*}
     By~\citep[Lemma~12]{abbasi2011improved}, $\det(\Lambda_{k+1})\leq 2\det(\Lambda_k)$ implies $\Lambda_{k+1} \preceq 2\Lambda_{k}$ and thus $\Lambda_{k}^{-1} \preceq 2\Lambda_{k+1}^{-1}$.
     Therefore, the first part is upper bounded by $\sqrt{2}\sum_{k, h} \beta \|\Phi(x_h^k,a_h^k)\|_{\Lambda_{k+1}^{-1}}\leq  \beta\sqrt{2T}\sqrt{\sum_{k, h}\|\Phi(x_h^k,a_h^k)\|^2_{\Lambda_{k+1}^{-1}}}$, by Cauchy-Schwarz inequality. Further invoking \cite[Lemma D.2]{jin2019provably}, we upper bound the last term by $\order\left(\beta\sqrt{T}\sqrt{\log \frac{\det(\Lambda_{T/H+1})}{\det(\Lambda_1)}}\right)=\order\left(\beta\sqrt{T}\sqrt{\log \left(\frac{\lambda+ 2T}{\lambda}\right)^d} \right)=\order\left(\beta\sqrt{dT\log T}\right)$. 
     For the second part, notice that since the event $\det(\Lambda_{k+1})>2\det(\Lambda_{k})$ cannot happen for more than $\order\left(\log \frac{\det(\Lambda_{T/H+1})}{\det(\Lambda_1)}\right)=\order\left(d\log T\right)$ times, this part is upper bounded by $\order\left(\beta dH\log T\right)$. 
     
     To conclude, we have shown that $\term_5=\order\left( \beta \sqrt{dT\log T} + \beta dH\log T + H\sqrt{T\log(1/\delta)}  \right)$ holds with probability at least $1-2\delta$.      
     Combining all the bounds with \pref{eq: discounted regret decompose}, we have 
     \begin{align*}
          \Reg_T &= \sum_{t=1}^T  \left(J^* - r(x_t,a_t)\right) = \order\left(  \frac{T\spv}{H} + \beta \sqrt{dT\log T} + \beta dH\log T + H\sqrt{T\log(1/\delta)} \right) \\
          &= \otil\left( \frac{T\spv}{H} + d^{3/2}H\sqrt{T} +  d^2H^2 \right) \tag{plug in the value of $\beta$}
     \end{align*}
     with probability at least $1-3\delta$. 
     Picking the optimal $H$ (the one specified in \pref{alg: optimistic qlearning disc}), we get that $\Reg_T = \otil\left(\sqrt{\spv }(dT)^{\frac{3}{4}} + \left(\spv d T \right)^{\frac{2}{3}}\right)$. 
\end{proof}

\begin{lemma}
    \label{lemma: connecting lemma}
    For any $x$, $|HJ^* - V_1^*(x)|\leq \spv$. 
\end{lemma}
\begin{proof}
    Let $\pi^*$ be the optimal policy of the infinite-horizon setting, and $(\pi_1, \ldots, \pi_H)$ be the optimal policy of the finite-horizon setting. 
    Without loss generality assume that both of them are deterministic policy.
    By the Bellman equation and the optimality of $\pi^*$, we have 
    \begin{align}
       v^*(x) &=\max_a \left(r(x, a) - J^* + \E_{x'\sim p(\cdot|x,a)}v^*(x)\right)  \label{eq: convertion 1} \\
                  &= r(x, \pi^*(x)) - J^* + \E_{x'\sim p(\cdot|x,\pi^*(x))}v^*(x).   \label{eq: convertion 2}
   \end{align}
    For any $x$, consider a state sequence $x_1=x, x_2,\ldots, x_H$ generated by $\pi^*$. By the suboptimality of $\pi^*$ in the finite-horizon setting, 
    \begin{align*}
         V_1^*(x)&\geq \E\left[\sum_{h=1}^H r(x_h, \pi^*(x_h))~\bigg|~ x_1=x, \ \ \pi^*\right] \\
         &=\E\left[\sum_{h=1}^H \left(J^* + v^*(x_h) - \E_{x'\sim p(\cdot|x_h,\pi^*(x_h))}[v^*(x')]\right)~\bigg|~ x_1=x, \ \ \pi^* \right]   \tag{by \pref{eq: convertion 2}}\\
         &=\E\left[\sum_{h=1}^H \left(J^* + v^*(x_h) -  v^*(x_{h+1})\right)~\bigg|~ x_1=x, \ \ \pi^* \right] \\
         &=HJ^* +  \E\left[v^*(x_1) -  v^*(x_{H+1})~\big|~ x_1=x, \ \ \pi^* \right] \\
         &\geq HJ^* - \spv. 
    \end{align*}
    Next, consider a state $x_1=x, x_2, \ldots, x_H$ generated by $(\pi_1,\ldots, \pi_H)$: 
    \begin{align*}
        V_1^*(x) 
        &= \E\left[\sum_{h=1}^H r(x_h,\pi_h(x_h))~\bigg|~ x_1=x,\ \ \{\pi_i\}_{i=1}^H  \right] \\
        &\leq \E\left[\sum_{t=1}^H \left( J^* + v^*(x_t) - \E_{x'\sim p(\cdot|x_h, \pi_h(x_h))}[v^*(x')] \right)~\bigg|~ x_1=x,\ \ \{\pi_i\}_{i=1}^H \right]     \tag{by \pref{eq: convertion 1}} \\
        &= \E\left[\sum_{t=1}^H \left( J^* + v^*(x_h) - v^*(x_{h+1}) \right)~\bigg|~ x_1=x,\ \ \{\pi_i\}_{i=1}^H \right] \\
        &= HJ^* + \E\left[v^*(x_1) -  v^*(x_{H+1})~\big|~ x_1=x, \ \ \{\pi_i\}_{i=1}^H \right] \\
        &\leq HJ^* + \spv. 
    \end{align*}
    Combining the two directions finishes the proof. 
\end{proof}

%% file: appendix-exp2.tex
\section{Omitted Analysis in \pref{sec: mdpexp2}}
\begin{figure}[H]
\includegraphics[width=\textwidth]{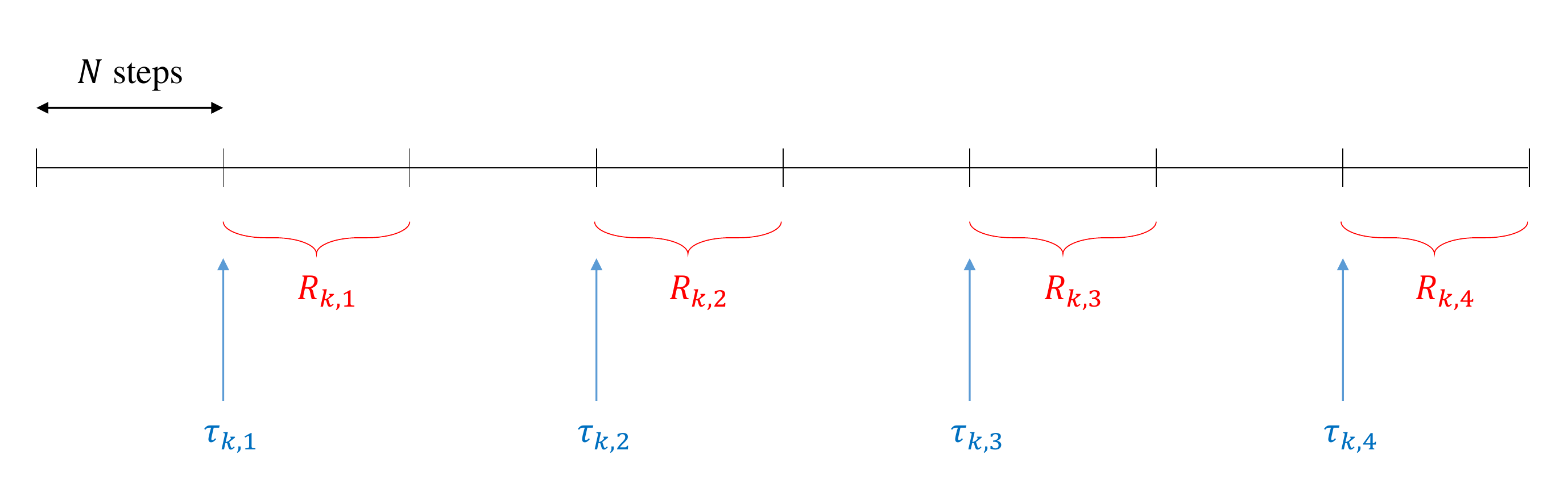} 
\caption{An illustration for the data collection process of \algExp. In the figure, we show how the algorithm collects $4$ trajectories of length $N$ (the red intervals) in an epoch with length $B=8N$.}
\label{fig: data collection}
\end{figure}

\pref{fig: data collection} is an illustration of the data collection scheme of \algExp.
Below, we first provide the proof for \pref{lem:uniformly mixing lemma}.

 \begin{proof}[Proof of \pref{lem:uniformly mixing lemma}]
Denote $\E[\cdot|x_1=x, a_t\sim \pi(\cdot|x_t), x_{t+1}\sim p(\cdot|x_t,a_t) \text{ for all $t\geq 1$}]$ by $\E[\cdot|x_1=x, \pi]$.  For any two initial states $u, u'\in\calX$, let $\delta_{u}$ and $\delta_{u'}$ be the Dirac measures with respect to $u$ and $u'$. Writing $\PP^\pi$ as $\PP$ for simplicity, we have for any time $t$, 
     \begin{align}
         &\left\vert \E\left[ r(x_t,a_t) ~|~ x_1=u, \pi \right] - \E\left[ r(x_t,a_t) ~|~ x_1=u', \pi \right]\right\vert \nonumber \\
         &=\left\vert \int_\calX \sum_{a\in\calA} \pi(a|x)r(x,a) \de \PP^{t-1} \delta_{u}(x)  - \int_\calX \sum_{a\in\calA} \pi(a|x)r(x,a) \de \PP^{t-1} \delta_{u'}(x) \right\vert \nonumber  \\
         &\leq 2\|\PP^{t-1} \delta_{u} - \PP^{t-1} \delta_{u'}\|_\TV \nonumber \\
         &\leq 2e^{-\frac{t-1}{\tmix}}\|  \delta_{u} -   \delta_{u'}\|_\TV  \tag{\pref{ass: uniform mixing}} \\
         &\leq 2e^{-\frac{t-1}{\tmix}}.   \label{eq: ergodic useful}
     \end{align}
Therefore, by the definition of $J^{\pi}(u)$ in \pref{sec: preliminaries}, we have
    \begin{align*}
        |J^\pi(u) - J^\pi(u')| \leq \lim_{T\rightarrow \infty} \frac{2}{T}\sum_{t=1}^T e^{-\frac{t-1}{\tmix}} = 0,
    \end{align*}
   proving that $J^{\pi}(u)$ is a fixed value independent of the initial state $u$ and can thus be denoted as $J^{\pi}$.

   Next, define the following two quantities: 
\begin{equation}\label{eq: v and q pi}
\begin{split}
     v_T^{\pi}(x) &= \E\left[ \sum_{t=1}^T \left(r(x_t,a_t) - J^{\pi}\right) ~\Big|~ x_1=x, \pi \right], \\
     q_T^{\pi}(x,a) &= \E\left[ \sum_{t=1}^T \left(r(x_t,a_t) - J^{\pi}\right) ~\Big|~ (x_1, a_1)=(x,a),  x_{t}\sim p(\cdot|x_{t-1},a_{t-1}),  a_t\sim \pi(\cdot|x_t) \text{\ for\ }t\geq 2 \right].
\end{split}
\end{equation}
We will show that $v^{\pi}(x)\triangleq \lim_{T\rightarrow \infty} v^{\pi}_T(x)$ and $q^{\pi}(x,a)\triangleq \lim_{T\rightarrow \infty} q^{\pi}_T(x,a)$ satisfy the conditions stated in \pref{lem:uniformly mixing lemma}. 
First we argue that they do exist.
Note that $J^{\pi}$ can be written as $\int_{\calX} \sum_{a} r(x,a) \pi(a|x) \de\nu^{\pi}(x)$ where $\nu^{\pi}$ is the stationary distribution under $\pi$.
Therefore, for any $T$, we have
\begin{align}
& \left\vert\E\left[ r(x_{T+1},a_{T+1}) - J^{\pi}  ~\Big|~ x_1=x, \pi \right]\right\vert \notag \\
& = \left\vert \int_\calX \sum_{a\in\calA} \pi(a|x')r(x',a) \de \PP^{T} \delta_{x}(x')  - \int_\calX \sum_{a\in\calA} \pi(a|x)r(x,a) \de \nu^{\pi}(x) \right\vert \notag \\
& \leq 2\|\PP^{T} \delta_{x} - \nu^\pi \|_\TV   \notag\\
& =  2\|\PP^{T} \delta_{x} - \PP^{T} \nu^\pi \|_\TV   \tag{by the definition of $\nu^\pi$}\\
& \leq 2e^{-\frac{T}{\tmix}}\|\delta_{x} - \nu^\pi \|_\TV \tag{by \pref{ass: uniform mixing}} \\
& \leq 2e^{-\frac{T}{\tmix}}, \label{eq:v_increment}
\end{align}
and thus
\[
|v_T^{\pi}(x)-v_{T+1}^{\pi}(x)|  = \left\vert\E\left[ r(x_{T+1},a_{T+1}) - J^{\pi}  ~\Big|~ x_1=x, \pi \right]\right\vert
\leq 2e^{-\frac{T}{\tmix}},
\]
which goes to zero and implies that $v^{\pi}(x) = \lim_{T\rightarrow \infty} v^{\pi}_T(x)$ exists.
On the other hand, by the definition we have 
     \begin{align*}
          q^{\pi}_T(x,a) = r(x,a) - J^{\pi} + \E_{x'\sim p(\cdot|x,a)} v^{\pi}_{T-1}(x'),
     \end{align*}
     and taking the limit on both sides shows that $q^{\pi}(x,a)= \lim_{T\rightarrow \infty} q^{\pi}_T(x,a)$ exists and satisfies the Bellman equation in the lemma statement:
     \begin{align*}
          q^{\pi}(x,a) = r(x,a) - J^{\pi} + \E_{x'\sim p(\cdot|x,a)} v^{\pi}(x').
     \end{align*} 
     Finally, \pref{eq:v_increment} also shows that 
     \[
     |v^{\pi}_T(x)|  \leq 2\sum_{t=1}^T e^{-\frac{t-1}{\tmix}} \leq  \frac{2}{1-e^{-\frac{1}{\tmix}}} \leq \frac{2}{1-\left(1-\frac{1}{2\tmix} \right)}= 4\tmix, \tag{using $e^{-x}\leq 1-\frac{1}{2}x$ for $x\in[0, 1]$ and $\tmix \geq 1$}
     \]
     and thus the range of $v^{\pi}$ is $[-4\tmix, 4\tmix]$ while the range of $q^{\pi}$ is $[-6\tmix, 6\tmix]$ since $|q^{\pi}(x,a)|\leq |r(x,a)| + |J^{\pi}|+\sup_{x'}|v^{\pi}(x')|\leq 2+4\tmix\leq 6\tmix$. 
     The last statement $\int_{\calX} v^{\pi}(x)\de\nu^{\pi}(x) = 0$ in the lemma is also clear since
$\int_{\calX} v_T^{\pi}(x)\de\nu^{\pi}(x) = 0$ for all $T$ by the equality $J^{\pi} = \int_{\calX} \sum_{a} r(x,a) \pi(a|x) \de\nu^{\pi}(x)$ and the fact that $x_1, \ldots, x_T$ all have marginal distribution $\nu^{\pi}$ when $x_1 =x$ is drawn from $\nu^{\pi}$.
\end{proof}

In \pref{sec: mdpexp2}, we mention that \pref{ass: linear policy value} is weaker than \pref{assump: Linear MDP} when \pref{ass: uniform mixing} holds. Below we provide a proof for this statement.

\begin{lemma}
\label{lemma:linear MDP implies linear value}
    Under \pref{ass: uniform mixing}, \pref{assump: Linear MDP} implies \pref{ass: linear policy value}. 
\end{lemma}

\begin{proof}
     Since \pref{ass: uniform mixing} holds, by \pref{lem:uniformly mixing lemma}, we have 
     \begin{align*}
         q^{\pi}(x,a) 
         &= r(x,a) - J^{\pi} + \E_{x'\sim p(\cdot|x,a)} v^{\pi}(x') \\
         &= \Phi(x,a)^\top \THE - J^{\pi} \Phi(x,a)^\top \e_1 + \Phi(x,a)^\top \int_{\calX} v^{\pi}(x') \de\MU(x') \tag{\pref{assump: Linear MDP}}\\
         &= \Phi(x,a)^\top \left(\THE - J^{\pi}\e_1 + \int_{\calX} v^{\pi}(x') \de\MU(x')\right). 
     \end{align*}
     Taking $w^{\pi}$ to be $\THE - J^{\pi}\e_1 + \int_{\calX} v^{\pi}(x') \de\MU(x')$ and noting that 
     $\|w^\pi\| \leq \|\THE\| + 1 + (\max_{x\in\calX} v^\pi(x)) \|\MU(\calX)\| \leq \sqrt{d} +1 + 4\tmix\sqrt{d}\leq 6\tmix\sqrt{d}$ finishes the proof. 
\end{proof}


\subsection{Proof of \pref{thm: MDPEXP2 bound}}
To prove \pref{thm: MDPEXP2 bound}, we first show a couple of useful lemmas.

\begin{lemma}
     \label{lem: closeness between Rkm and q}
     Let $k$ be any number in $\{1, 2, \ldots, \frac{T}{B}\}$ and $m$ be any number in $\{1,2,\ldots, \frac{B}{2N}\}$. Let $\E[\cdot~|~ \tau_{k,m}]$ denote the expectation conditioned on $(x_{\tau_{k,m}}, a_{\tau_{k,m}})$ and all history before time $\tau_{k,m}$ (recall the definitions of  $\tau_{k,m}$ and $R_{k,m}$ in \pref{alg: mdpexp2}). Then we have 
     \begin{align*}
          \left|\E[R_{k,m}~|~\tau_{k,m}] - \left(q^{\pi_k}(x_{\tau_{k,m}}, a_{\tau_{k,m}}) + NJ^{\pi_k}\right)\right| \leq \frac{1}{T^7}. 
     \end{align*}
\end{lemma}
\begin{proof}
    Recalling the definition of $q_N^{\pi_k}$ in \pref{eq: v and q pi}, we have
    \begin{align}
    &\E[R_{k,m}~|~\tau_{k,m}] \nonumber\\
    &= \E\left[\sum_{t=1}^{N} r(x_t,a_t)~\Big|~ (x_1, a_1) = (x_{\tau_{k,m}}, a_{\tau_{k,m}}), \ \  x_{t}\sim p(\cdot|x_{t-1}, a_{t-1}), \ \  a_t\sim \pi_k(\cdot|x_t) \text{\ for\ }t\geq 2 \right]   \nonumber  \\
    &= q_N^{\pi_k}(x_{\tau_{k,m}}, a_{\tau_{k,m}}) + NJ^{\pi_k}. \label{eq: combine q and qN}
    \end{align}
    Then we bound the difference between $q_N^{\pi}(x,a)$ and $q^{\pi}(x,a)$ (which is $\lim_{N\rightarrow\infty}q_N^{\pi}(x,a)$ as shown in the proof of \pref{lem:uniformly mixing lemma}) for any $\pi, x,a$:
    \begin{align*}
         &|q_N^{\pi}(x,a) - q^{\pi}(x,a)|  \\
         &= \left|\E\left[\sum_{t=N+1}^\infty (r(x_t,a_t)-J^{\pi})~\Big|~ (x_1, a_1)=(x,a), x_{t}\sim p(\cdot|x_{t-1}, a_{t-1}), \ \  a_t\sim \pi_k(\cdot|x_t) \text{\ for\ }t\geq 2 \right] \right| \\
         &\leq 2\sum_{t=N+1}^{\infty} e^{-\frac{t-1}{\tmix}} \leq \frac{2e^{-\frac{N}{\tmix}}}{1-e^{-\frac{1}{\tmix}}} \leq 4\tmix e^{-\frac{N}{\tmix}}.    \tag{\pref{eq:v_increment}}
    \end{align*}
    Recall that $N=8\tmix \log T$, and without loss of generality we assum $\tmix\leq T/4$ (otherwise the regret bound is vacuous). Thus we can bound the last expression by $\frac{4\tmix}{T^8}\leq \frac{1}{T^7}$. Combining this with \pref{eq: combine q and qN} finishes the proof. 
\end{proof}


    

\begin{lemma}
\label{lemma: wk and theta}
Let $\E_k[\cdot]$ denote the expectation conditioned on all history before epoch $k$. Then
\begin{align*}
    \left\|\E_k[w_k] - \left(w^{\pi_k} + NJ^{\pi_k}\e_1 \right)\right\| \leq \frac{1}{T^2}.  
\end{align*}
\end{lemma}
\begin{proof}
     Let $I_k = \one[\lambda_{\min}(M_k)\geq  \frac{B\sigma}{24N}]$. We proceed as follows:
    \begin{align*}
        \E_k[w_k] 
        &= \E_k\left[ I_k M_k^{-1} \sum_{m=1}^{\frac{B}{2N}} \Phi(x_{\tau_{k,m}},a_{\tau_{k,m}}) R_{k,m}\right] \tag{definition of $w_k$} \\
        &= \E_k\left[ I_k M_k^{-1} \sum_{m=1}^{\frac{B}{2N}} \Phi(x_{\tau_{k,m}},a_{\tau_{k,m}})\E_k[R_{k,m}|x_{\tau_{k,m}}, a_{\tau_{k,m}}] \right]    \tag{taking expectation for $R_{k,m}$ conditioned on $(x_{\tau_{k,m}}, a_{\tau_{k,m}})$}\\
         &= \E_k\left[ I_k M_k^{-1} \sum_{m=1}^{\frac{B}{2N}}\Phi(x_{\tau_{k,m}},a_{\tau_{k,m}}) \left(q^{\pi_k}(x_{\tau_{k,m}},a_{\tau_{k,m}}) + NJ^{\pi_k}\right)  \right] \\
        &\qquad \qquad  + \E_k\left[ I_k M_k^{-1} \sum_{m=1}^{\frac{B}{2N}} \Phi(x_{\tau_{k,m}}, a_{\tau_{k,m}})\epsilon_k(x_{\tau_{k,m}}, a_{\tau_{k,m}}) \right]  \tag{define $\epsilon_k(x_{\tau_{k,m}}, a_{\tau_{k,m}}) = \E_k[R_{k,m}|x_{\tau_{k,m}}, a_{\tau_{k,m}}] - \left(q^{\pi_k}(x_{\tau_{k,m}},a_{\tau_{k,m}}) + NJ^{\pi_k}\right)$}  \\
        &= \E_k\left[ I_k M_k^{-1} \sum_{m=1}^{\frac{B}{2N}}\Phi(x_{\tau_{k,m}},a_{\tau_{k,m}})\Phi(x_{\tau_{k,m}},a_{\tau_{k,m}})^\top \left(w^{\pi_k} + NJ^{\pi_k}\e_1\right)  \right] \\
        &\qquad \qquad  + \E_k\left[I_k M_k^{-1} \sum_{m=1}^{\frac{B}{2N}} \Phi(x_{\tau_{k,m}}, a_{\tau_{k,m}})\epsilon_k(x_{\tau_{k,m}}, a_{\tau_{k,m}}) \right]   \tag{by \pref{ass: linear policy value}} \\
        &= \E_k\left[ I_k M_k^{-1} \sum_{m=1}^{\frac{B}{2N}} \sum_{a} \pi_k(a|x_{\tau_{k,m}}) \Phi(x_{\tau_{k,m}},a)\Phi(x_{\tau_{k,m}},a)^\top \left(w^{\pi_k} + NJ^{\pi_k}\e_1\right)  \right] \\
        &\qquad \qquad  +\E_k\left[I_k M_k^{-1} \sum_{m=1}^{\frac{B}{2N}} \sum_{a}\pi_k(a|x_{\tau_{k,m}})\Phi(x_{\tau_{k,m}}, a)\epsilon_k(x_{\tau_{k,m}}, a) \right]  \tag{taking expectation for $a_{\tau_{k,m}}$ conditioned on $x_{\tau_{k,m}}$} \\
        &= \E_k\left[ I_k  \left(w^{\pi_k} + NJ^{\pi_k}\e_1\right)  \right] + \bm{\epsilon} \tag{define $\bm{\epsilon}=\E_k\left[I_k M_k^{-1} \sum_{m=1}^{\frac{B}{2N}} \sum_a\pi_k(a|x_{\tau_{k,m}})\Phi(x_{\tau_{k,m}}, a)\epsilon_k(x_{\tau_{k,m}}, a) \right]$} \\
        &= w^{\pi_k} + NJ^{\pi_k}\e_1 - \E_k\left[(1-I_k)(w^{\pi_k}+NJ^{\pi_k}\e_1)\right]+ \bm{\epsilon}. 
    \end{align*}
    By \pref{lem: closeness between Rkm and q}, we have $|\epsilon_k(x_{\tau_{k,m}}, a_{\tau_{k,m}})| \leq 1/T^7$ and thus 
    \begin{align*}
        \|\bm{\epsilon}\| 
        &\leq \E_k \left[\sum_{m=1}^{\frac{B}{2N}}\left\| I_k  M_k^{-1} \sum_{a}\pi_k(a|x_{\tau_{k,m}})\Phi(x_{\tau_{k,m}}, a)\epsilon_k(x_{\tau_{k,m}}, a_{\tau_{k,m}})\right\|\right]\\
&= \E_k \left[\sum_{m=1}^{\frac{B}{2N}}\left\| I_k  M_k^{-1} \sum_{a}\pi_k(a|x_{\tau_{k,m}})\Phi(x_{\tau_{k,m}}, a)\Phi(x_{\tau_{k,m}}, a)^\top \e_1 \epsilon_k(x_{\tau_{k,m}}, a_{\tau_{k,m}})\right\|\right]\\
        &= \E_k\left[\sum_{m=1}^{\frac{B}{2N}}\left\| I_k  \e_1 \epsilon_k(x_{\tau_{k,m}}, a_{\tau_{k,m}}) \right\|\right]
        \leq \E_k\left[\sum_{m=1}^{\frac{B}{2N}}\frac{1}{T^7}\right]\leq \frac{1}{T^6}. 
    \end{align*}
    On the other hand, we also have
  \begin{align*}
\left\|\E_k\left[(1-I_k)(w^{\pi_k}+NJ^{\pi_k}\e_1)\right]\right\| \leq \E_k\left[(1-I_k)\right] (6\tmix\sqrt{d}+N)
\leq \frac{6\tmix \sqrt{d} + N}{T^3}.
\end{align*}  
where the last step is by \pref{lemma: least eigen} (stated after this proof).
    Finally, combining everything proves
    \begin{align*}
        \left\|\E_k[w_k] - (w^{\pi_k} + NJ^{\pi_k}\e_1)\right\|\leq \frac{1}{T^6} + \frac{6\tmix \sqrt{d} + N}{T^3} \leq \frac{1}{T^2},    
    \end{align*}
    where we assume $6\tmix \sqrt{d} + N = 6\tmix \sqrt{d} + 8\tmix\log T$ is at most $\frac{T}{2}$ (otherwise the regret bound is vacuous). 
\end{proof}

\begin{lemma}
\label{lemma: least eigen}
For any $k \in \{1, \ldots, T/B\}$, conditioning on the history before epoch $k$, we have with probability at least $1-\frac{1}{T^3}$, 
$\lambda_{\min}(M_k)\geq  \frac{B\sigma}{24N}$. 
\end{lemma}

\begin{proof}
    We consider a fixed $k$. Notice that since $N$ is larger than $\tmix$, the state distribution at $\tau_{k,m}$ conditioned on all trajectories collected before (which all happen before $\tau_{k,m}-N$) would be close to the stationary distribution $\nu^{\pi_k}$. For the purpose of analysis, we consider an \emph{imaginary world} where all history before epoch $k$ remains the same as the real world, but in epoch $k$, at time $t=\tau_{k,m}$, $\forall m=1,2,\ldots$, the state distribution is reset according to the stationary distribution, i.e., $x_{\tau_{k,m}}\sim \nu^{\pi_k}$; for other rounds, it follows the state transition driven by $\pi_k$, the same as the real world. we denote the expectation (given the history before epoch $k$) in the imaginary world as $\E'_k[\cdot]$. 
     
     Fro simplicity, define $y_m=x_{\tau_{k,m}}$, $z_m=\{a_{\tau_{k,m}}, R_{k,m}\}$ and $m^*=\frac{B}{2N}$. Note that $M_k$ is a function of $\{y_m\}_{m=1}^{m^*}$ and that $(y_{i-1}, z_{i-1})\rightarrow  y_i \rightarrow z_i$ form a Markov chain.      
     Therefore, by writing $M_k=M_k\left(y_1, \ldots, y_{m^*}\right)$, and considering any function $f$ of $M_k$, we have 
     \begin{align*}
         \E_k[f(M_k)] &= \int f\left(M_k\left(y_1, \ldots, y_{m^*}\right)\right) \de q(y_1)\de q(z_1|y_1)  \de q(y_2|y_1,z_1)  \de q(z_2|y_2) \cdots \\
         &\hspace{15em} \de q(y_{m^*}|y_{m^*-1}, z_{m^*-1})  \de q(z_{m^*}|y_{m^*})  
     \end{align*}
     and 
      \begin{align*}
         \E_k'[f(M_k)] = \int f\left(M_k\left(y_1, \ldots, y_{m^*}\right)\right)  \de q'(y_1) \de q(z_1|y_1) \de q'(y_2) \de q(z_2|y_2) \cdots  \de q'(y_{m^*}) \de q(z_{m^*}|y_{m^*})  
     \end{align*}
     where $q$ and $q'$ denote the probability measure in the real and the imaginary worlds respectively (conditioned on the history before epoch $k$). Note that by our construction, in the imaginary world $y_i$ is independent of $(y_1, z_1, \ldots, y_{i-1}, z_{i-1})$, while $z_i|y_i$ follows the same distribution as in the real world. 
     By the uniform-mixing assumption, we have that 
     \begin{align*}
         \| q'(y_m) - q(y_m|y_{m-1}, z_{m-1}) \|_{\text{TV}}\leq e^{-\frac{N}{\tmix}}\leq \frac{1}{T^8}, 
     \end{align*}
     implying that 
     \begin{align}
         \left| \E_k[f(M_k)] - \E_k'[f(M_k)]\right| \leq \frac{2}{T^8}\times \frac{B}{2N} \times f_{\max}\leq \frac{f_{\max}}{T^7},     \label{eq: max expectation diff}
     \end{align}
     where $f_{\max}$ is the maximum magnitude of $f(\cdot)$. 
     Picking $f(M) = \one\left[\lambda_{\min}(M)\leq \frac{B\sigma}{24N}\right]$ (with $f_{\max} = 1$ clearly), we have shown that
     \[
     {\Pr}_k\left[\lambda_{\min}\left(M_k\right)\leq  \frac{B\sigma }{24N}\right] 
     \leq {\Pr}'_k\left[\lambda_{\min}\left(M_k\right)\leq  \frac{B\sigma }{24N}\right] + \frac{1}{T^7}.
     \]
    It remains to bound ${\Pr}'_k\left[\lambda_{\min}\left(M_k\right)\leq  \frac{B\sigma }{24N}\right]$.
    Notice that \[\E_k'[M_k]= \frac{B}{2N}\times \int_{\calX} \sum_{a}\pi_k(a|x)\Phi(x,a)\Phi(x,a)^\top \de\nu^{\pi_k}(x) \succeq \frac{B}{2N} \times \sigma I\] by \pref{ass: assump A4}. 
    Using standard matrix concentration results (specifically, \pref{lemma: matrix concentration} with $\delta=\frac{11}{12}$, $n=\frac{B}{2N}=\frac{16}{\sigma}\log(dT)$, $X_m=\sum_{a}\pi_k(a|x_{\tau_{k,m}})\Phi(x_{\tau_{k,m}}, a)\Phi(x_{\tau_{k,m}}, a)^\top$, $R=2$, and $r = \frac{B\sigma}{2N} = 16\log(dT)$), we get
    \begin{align*}
        {\Pr}'_k\left[\lambda_{\min}\left(M_k\right)\leq \frac{1}{12}\times \frac{B\sigma }{2N}\right] 
        &\leq d\cdot \exp\left(-\frac{121}{144}\times 16\log (dT)\times \frac{1}{4}\right) \\
        &\leq d\cdot \exp\left(-3.3\log(dT)\right)\leq \frac{1}{T^{3.3}}. 
    \end{align*}
    In other words, we have shown
    \begin{align*}
    {\Pr}_k\left[\lambda_{\min}\left(M_k\right)\leq  \frac{B\sigma }{24N}\right] 
    \leq \frac{1}{T^{3.3}} + \frac{1}{T^{7}} \leq \frac{1}{T^{3}},
    \end{align*}
    which completes the proof.
\end{proof}

\begin{lemma}(Theorem 2 in \cite{harvey2015})
\label{lemma: matrix concentration}
   Let $X_1, \ldots, X_n$ be independent, random, symmetric, real matrices of size $d\times d$ with $0\preceq X_m \preceq R I$ for all $m$. Suppose $r I \preceq \E[\sum_{m=1}^n X_m]$ for some $r>0$. Then for all $\delta\in [0,1]$, one has
   \begin{align*}
       \Pr\left[\lambda_{\min}\left(\sum_{m=1}^n X_m\right)\leq (1-\delta)r \right] \leq d\cdot e^{-\delta^2 r / (2R)}. 
   \end{align*}
\end{lemma}

\begin{lemma}
    \label{lemma: perstate regret}
    With $\eta\leq \frac{\sigma}{24N}$, \algExp guarantees for all $x$:
    \begin{align*}
        \E\left[\sum_{k=1}^{T/B}\sum_{a}\left(\pi^*(a|x)-\pi_k(a|x)\right)q^{\pi_k}(x,a)\right] \leq  \order\left(\frac{\ln |\calA|}{\eta} + \eta \frac{TN^2}{B\sigma} \right). 
    \end{align*}
\end{lemma}
\begin{proof}
    Note that by the definition of $w_k$ we have
    \begin{align}
     |w_k^\top\Phi(x,a)|\leq \sqrt{2}\eta \|w_t\|  
    \leq \sqrt{2}\eta  \times \frac{24N}{B\sigma}  \times \frac{B}{2N}\times \sqrt{2}N = \frac{24N}{\sigma}, \label{eq:loss_bound}
    \end{align}
    and thus $\eta |w_k^\top\Phi(x,a)| \leq 1$ by our choice of $\eta$.
    Therefore, using the standard regret bound of exponential weight (see e.g., \cite[Theorem 1]{bubeck2012towards}), we have
    \begin{align}
        \sum_{k=1}^{T/B}\sum_{a}\left(\pi^*(a|x)-\pi_k(a|x)\right) \left(w_k^\top \Phi(x,a)\right) \leq \order\left(\frac{\ln |\calA|}{\eta} + \eta \sum_{k=1}^{T/B}\sum_{a} \pi_k(a|x) \left(w_k^\top \Phi(x,a)\right)^2 \right). \label{eq: regret bound single state}  
    \end{align}
    Taking expectation, the left-hand side becomes
    \begin{align*}
        &\E\left[\sum_{k=1}^{T/B}\sum_{a}\left(\pi^*(a|x)-\pi_k(a|x)\right) \left(w_k^\top \Phi(x,a)\right)\right]\\
        &=\E\left[\sum_{k=1}^{T/B}\sum_{a}\left(\pi^*(a|x)-\pi_k(a|x)\right) \left((w^{\pi_k} + NJ^{\pi_k}\e_1) \cdot \Phi(x,a)\right)\right] - \order(1) \tag{\pref{lemma: wk and theta}} \\
        &=\E\left[\sum_{k=1}^{T/B}\sum_{a}\left(\pi^*(a|x)-\pi_k(a|x)\right) ({w^{\pi_k}}^\top \Phi(x,a) + NJ^{\pi_k}) \right] - \order(1) \\
        &=\E\left[\sum_{k=1}^{T/B}\sum_{a}\left(\pi^*(a|x)-\pi_k(a|x)\right) {w^{\pi_k}}^\top \Phi(x,a)\right] - \order(1) \\
        &=\E\left[\sum_{k=1}^{T/B}\sum_{a}\left(\pi^*(a|x)-\pi_k(a|x)\right) q^{\pi_k}(x,a)\right] - \order(1).  \tag{\pref{ass: linear policy value}}
    \end{align*}
    To bound the expectation of the right-hand side of \pref{eq: regret bound single state},
    we focus on the key term $\E_k\left[\sum_{a}\pi_k(a|x)(w_k^\top \Phi(x,a))^2\right]$ ($\E_k$ denotes the expectation conditioned on the history before epoch $k$) and use the same argument as done in the proof of \pref{lemma: least eigen} via the help of an imaginary word where everything is the same as the real world except that the first state of each trajectory $x_{\tau_{k,m}}$ for $m=1,2,\ldots, B/2N$ is reset according to the stationary distribution $\nu^{\pi_k}$ ($\E_k'$ denotes the conditional expectation in this imaginary world).
By the exact same argument ({\it cf.} \pref{eq: max expectation diff}), we have
\[
\E_k\left[\sum_{a}\pi_k(a|x)(w_k^\top \Phi(x,a))^2\right]
\leq \E_k'\left[\sum_{a}\pi_k(a|x)(w_k^\top \Phi(x,a))^2\right] + \frac{B}{T^8N} \times \left(\frac{24N}{\sigma}\right)^2
\]
where we use the range of $(w_k^\top \Phi(x,a))^2$ derived earlier in \pref{eq:loss_bound}.
It remains to bound $\E_k'\left[\sum_{a}\pi_k(a|x)(w_k^\top \Phi(x,a))^2\right]$, which we proceed as follows with $I_k = \one[\lambda_{\min}(M_k)\geq  \frac{B\sigma}{24N}]$:
    \begin{align*}
        &\E_k'\left[\sum_{a}\pi_k(a|x)(w_k^\top \Phi(x,a))^2\right] \\
        &=\E_k'\left[\sum_{a}\pi_k(a|x)\left(\Phi(x,a)^\top M_k^{-1} \sum_{m=1}^{\frac{B}{2N}} \Phi(x_{\tau_{k,m}}, a_{\tau_{k,m}})R_{k,m} \right)^2 I_k \right] \\
        &\leq N^2\E_k'\left[\sum_{a}\pi_k(a|x)\left(\Phi(x,a)^\top M_k^{-1} \sum_{m=1}^{\frac{B}{2N}} \Phi(x_{\tau_{k,m}}, a_{\tau_{k,m}}) \right)^2 I_k\right] \tag{$R_{k,m} \leq N$}\\
        &=N^2\E_k'\left[\sum_{a}\pi_k(a|x)\Phi(x,a)^\top M_k^{-1} \left(\sum_{m=1}^{\frac{B}{2N}} \Phi(x_{\tau_{k,m}}, a_{\tau_{k,m}})\right)\left(\sum_{m=1}^{\frac{B}{2N}} \Phi(x_{\tau_{k,m}}, a_{\tau_{k,m}})\right)^\top M_k^{-1}\Phi(x,a) I_k\right] \\
        &\leq \frac{BN}{2}\E_k'\left[\sum_{a}\pi_k(a|x)\Phi(x,a)^\top M_k^{-1} \left(\sum_{m=1}^{\frac{B}{2N}} \Phi(x_{\tau_{k,m}}, a_{\tau_{k,m}})\Phi(x_{\tau_{k,m}}, a_{\tau_{k,m}})^\top\right) M_k^{-1}\Phi(x,a)I_k \right] \tag{Cauchy-Schwarz inequality}\\
        &= \frac{BN}{2}\E_k'\left[\sum_{a}\pi_k(a|x)\Phi(x,a)^\top M_k^{-1} \left(\sum_{m=1}^{\frac{B}{2N}}\sum_{a'} \pi_k(a'|x_{\tau_{k,m}})\Phi(x_{\tau_{k,m}}, a')\Phi(x_{\tau_{k,m}}, a')^\top\right) M_k^{-1}\Phi(x,a)I_k \right] \\
        &=\frac{BN}{2}\E_k'\left[\sum_{a}\pi_k(a|x)\Phi(x,a)^\top M_k^{-1} \Phi(x,a)I_k \right] \\
        &\leq \order\left(BN \times \frac{N}{B\sigma}\right) \tag{definition of $I_k$}\\
        &= \order\left(\frac{N^2}{\sigma}\right). 
    \end{align*}
    Combining everything shows 
     \begin{align*}
        \E\left[\sum_{k=1}^{T/B}\sum_{a}\left(\pi^*(a|x)-\pi_k(a|x)\right)q^{\pi_k}(x,a)\right] 
        &\leq  \order\left(\frac{\ln |\calA|}{\eta} + \eta \frac{T}{B}\left( \frac{N^2}{\sigma} + \frac{NB}{T^8 \sigma} \right)\right) \\
        &\leq  \order\left(\frac{\ln |\calA|}{\eta} + \eta \frac{TN^2}{B\sigma} \right), 
    \end{align*}
    which finishes the proof.
\end{proof}


We are now ready to prove \pref{thm: MDPEXP2 bound}.
\begin{proof}[Proof of \pref{thm: MDPEXP2 bound}]
First, decompose the regret as:
\begin{align*}
    &\Reg_T=\E\left[\sum_{t=1}^T (J^* - r(x_t,a_t)) \right] \\
    &=\E\left[\sum_{k=1}^{T/B} B(J^*-J^{\pi_k})\right] + \E\left[ \sum_{k=1}^{T/B} \sum_{t=(k-1)B+1}^{kB} (J^{\pi_k}-r(x_t,a_t))\right]. 
\end{align*}
For the first term above, we apply the value difference lemma (see e.g., \citep[Lemma~15]{wei2020model}): 
\begin{align*}
    &\E\left[\sum_{k=1}^{T/B} B(J^*-J^{\pi_k})\right]\\
    &=\E\left[\sum_{k=1}^{T/B}B \int_{\calX}\sum_{a} (\pi^*(a|x)-\pi_k(a|x))q^{\pi_k}(x,a)\de \nu^{\pi^*}(x)\right]   \\
    &= O\left(\frac{B\ln |\calA|}{\eta} + \eta \frac{TN^2}{\sigma}\right).   \tag{by \pref{lemma: perstate regret}}
\end{align*}

For the second term, we first consider a specific $k$: 
\begin{align*}
    &\E_k\left[\sum_{t=(k-1)B+1}^{kB} (J^{\pi_k}-r(x_t,a_t))\right] \\
    &=\E_k \left[\sum_{t=(k-1)B+1}^{kB} (\E_{x'\sim p(\cdot|x_t,a_t)}[v^{\pi_k}(x')] - q^{\pi_k}(x_t,a_t))\right] \tag{Bellman equation}\\
    &=\E_k \left[\sum_{t=(k-1)B+1}^{kB} (v^{\pi_k}(x_{t+1}) - v^{\pi_k}(x_t))\right] \\
    &= v^{\pi_k}(x_{kB+1}) - v^{\pi_k}(x_{(k-1)B+1}). 
\end{align*}
Therefore, 
\begin{align}
     &\E\left[ \sum_{k=1}^{T/B} \sum_{t=(k-1)B+1}^{kB} (J^{\pi_k}-r(x_t,a_t))\right]  \nonumber \\
     &\leq \E\left[ \sum_{k=1}^{T/B}  \left(v^{\pi_k}(x_{kB+1}) - v^{\pi_k}(x_{(k-1)B+1})\right) \right]   \nonumber  \\
     &\leq \E\left[ \sum_{k=2}^{T/B}  \left(v^{\pi_{k-1}}(x_{(k-1)B+1}) - v^{\pi_k}(x_{(k-1)B+1})\right) \right] + \order(\tmix). \label{eq: MDPEXP2 stability}
\end{align}
We bound the last summation using the fact that $\pi_k$ and $\pi_{k-1}$ are close. 
Indeed, by the update rule of the algorithm, we have 
\begin{align*}
     \pi_k(a|x) - \pi_{k-1}(a|x) 
     &=  \frac{\pi_{k-1}(a|x)e^{\eta \Phi(x,a)^\top w_{k-1}}}{\sum_{b\in\calA} \pi_{k-1}(b|x)e^{\eta \Phi(x,b)^\top w_{k-1}} } - \pi_{k-1}(a|x) \\
     &\leq \frac{\pi_{k-1}(a|x)e^{\eta \Phi(x,a)^\top w_{k-1}}}{\sum_{b\in\calA} \pi_{k-1}(b|x)}e^{-\min_b \eta \Phi(x,b)^\top w_{k-1}} - \pi_{k-1}(a|x) \\
     &\leq \pi_{k-1}(a|x) \left(e^{2\eta \max_b |\Phi(x,b)^\top w_{k-1}| } -1 \right).
\end{align*}
Recall that in the proof of \pref{lemma: perstate regret}, we have shown $\eta \max_b |\Phi(x,b)^\top w_{k-1}|\leq 1$ as long as $\eta \leq \sigma/(24N)$. Combining with the fact $e^{2x}\leq 1+8x$ for $x\in[0,1]$ we have
\begin{align*}
     \left(e^{2\eta \max_b |\Phi(x,b)^\top w_{k-1}| } -1 \right) \leq 8\eta \max_b |\Phi(x,b)^\top w_{k-1}| = \order\left(\eta\times \frac{N}{\sigma}\right),
\end{align*}
where the last step is by \pref{eq:loss_bound}.
This shows
\begin{align*}
    \pi_k(a|x) - \pi_{k-1}(a|x)  \leq \order\left(\frac{\eta N}{\sigma}\pi_{k-1}(a|x)\right). 
\end{align*}
Similarly, we can show $\pi_{k-1}(a|x) - \pi_{k}(a|x) = \order\left(\frac{\eta N}{\sigma}\pi_{k-1}(a|x)\right)$ as well. 
By the same argument of~\citep[Lemma 7]{wei2020model} (summarized in \pref{lem: stability of V} for completeness), this implies:
\begin{align*}
    |v^{\pi_{k}}(x) - v^{\pi_{k-1}}(x)| \leq \order\left(\eta \frac{N^3}{\sigma} + \frac{1}{T^2}\right). 
\end{align*}
for all $x$. 
Continuing from \pref{eq: MDPEXP2 stability}, we arrive at
\begin{align*}
    \E\left[ \sum_{k=1}^{T/B} \sum_{t=(k-1)B+1}^{kB} (J^{\pi_k}-r(x_t,a_t))\right] = \order\left(\eta \frac{T}{B}\frac{N^3}{\sigma} +\tmix\right). 
\end{align*}
Combining everything, we have shown
\begin{align}
    \Reg_T &= \order\left(\frac{B\ln |\calA|}{\eta} + \eta \frac{TN^2}{\sigma} + \eta \frac{TN^3}{B\sigma} +  \tmix \right) \nonumber \\
    &= \otil\left(\frac{\tmix}{\sigma\eta} + \eta \frac{T\tmix^2}{\sigma} \right) \tag{definition of $N$ and $B$}\\
    &= \otil\left(\frac{1}{\sigma}\sqrt{\tmix^3 T}\right),   \tag{by the choice of $\eta$ specified in \pref{alg: mdpexp2}} 
\end{align}
which finishes the proof.
\end{proof}

\begin{lemma}\label{lem: stability of V}
     If $\pi'$ and $\pi$ satisfy $|\pi'(a|x)-\pi(a|x)|\leq \order\left(\beta \pi(a|x)\right)$ for all $x,a$ and some $\beta>0$, and $N\geq 4\tmix\log T$, then $
          |v^{\pi'}(x) - v^{\pi}(x)| \leq \order(\eta\beta N^2 + \frac{1}{T^2})$. 
\end{lemma}
\begin{proof}
    See the proof of \cite[Lemma 7]{wei2020model}. 
\end{proof}

\begin{remark}\label{rem:sigma}
\textup{
Notice that by the definition of $\sigma$, 
\begin{align*}
    \sigma
    &\leq \lambda_{\min}  \left(\int_{\calX}\left( \sum_{a}\pi(a|x)\Phi(x,a)\Phi(x,a)^\top\right) \de\nu^{\pi}(x)\right) \\
    &\leq \frac{1}{d}\text{trace}\left[\int_{\calX}\left( \sum_{a}\pi(a|x)\Phi(x,a)\Phi(x,a)^\top\right) \de\nu^{\pi}(x)\right] \\
    &\leq \frac{1}{d}\int_{\calX}\left( \sum_{a}\pi(a|x)\|\Phi(x,a)\|^2\right) \de\nu^{\pi}(x)    \tag{$\text{trace}[\Phi(x,a)\Phi(x,a)^\top] = \|\Phi(x,a)\|^2$} \\
    &\leq \frac{2}{d},   \tag{$\|\Phi(x,a)\|^2\leq 2$ by \pref{assump: Linear MDP}}
\end{align*}
which implies $\frac{1}{\sigma}\geq \frac{d}{2}$. 
Therefore, the regret bound in \pref{thm: MDPEXP2 bound} has an implicit $\Omega(d)$ dependence. 
}
\end{remark}

%% file: appendix-NPG.tex

\section{\algExp with unknown $\tmix$ and $\sigma$}
\label{app: unknown param}
To execute \algExp when $\tmix$ and $\sigma$ are unknown, we propose to use \emph{doubling trick} on $T$, and let $N=T^{0.4\xi}$, $B=T^{0.8\xi}$ for some $0<\xi<1$. More precisely, consider the following algorithm. 
\begin{algorithm}
     \caption{}
     \For{$i=0, 1, 2, \ldots$}{
          $W\leftarrow 64\cdot 2^i$ \\
          Execute \pref{alg: mdpexp2} for $W$ steps, with parameters $N=W^{0.4\xi}$, $B=W^{0.8\xi}$, $\eta=\sqrt{1/(NW)}$, and the condition in \pref{line: check lambda min} replaced by $\lambda_{\min}(M_k)\geq \frac{4}{3}\log (dW)$. 
     }
\end{algorithm}

To get a regret bound for this algorithm, we focus on a time interval that corresponds to a specific $i$.
Note that as long as $N\geq 8\tmix \log W$ and $B\geq 32N\log (dW)\sigma^{-1}$ (i.e., when the values of $N$ and $B$ are larger than the required values as specified in \pref{alg: mdpexp2}), we can redo the analysis of \pref{lemma: perstate regret} and \pref{thm: MDPEXP2 bound} (details omitted), and get 
\begin{align*}
    \Reg_i = \otil\left( \frac{B\ln |\calA|}{\eta} + \eta WBN + N \right). 
\end{align*}
as the regret for this interval. With the choice of $\eta$, we get $\Reg_i = \otil\left(B\sqrt{NW}\right)=\otil\left(W^{\frac{1}{2} + \xi}\right)$, where $W = 64 \cdot 2^i$. 

On the other hand, notice that the condition $N= W^{0.4\xi}\geq 8\tmix \log W$ holds except for a constant number of steps (the constant depends on $\xi$ and $\tmix$). Similarly, the condition $B\geq 32N\log (dW)\sigma^{-1}$ holds as long as $W^{0.4\xi} \geq 32 \log(dW)\sigma^{-1}$, which also holds except for a constant number of steps (the constant depends on $\xi$ and $\sigma$). 

Overall, we get an asymptotic regret bound of $\otil\left(T^{\frac{1}{2} + \xi}\right)$ except for a constant number of steps. The choice of $\xi$ trades the asymptotic performance with the constant regret term. 

\section{Connection between Natural Policy Gradient and \algExp}
\label{app: NPG}

The connection between the exponential weight algorithm \cite{freund1995desicion} and the classic natural policy gradient (NPG) algorithm \cite{kakade2002natural} under softmax parameterization has been discussed in \cite{agarwal2019optimality}. Further connections between exponential weight algorithms and several relative-entropy-regularized policy optimization algorithms (e.g., TRPO \cite{schulman2015trust}, A3C \cite{mnih2016asynchronous}, PPO \cite{schulman2017proximal}) are also drawn in \cite{neu2017unified}. In this section, we review these connection, and argue that because of the different way of constructing the policy gradient estimator, our \algExp is actually more sample efficient than the version of NPG discussed in \cite{agarwal2019optimality} under the setting considered in \pref{sec: mdpexp2}.    

\subsection{Equivalence between NPG with softmax parameterization and exponential weight updates}
We first restates \cite[Lemma 5.1]{agarwal2019optimality}, which shows that NPG with softmax parameterization is equivalent to exponential weight updates: 
\begin{lemma}[Lemma 5.1 of \cite{agarwal2019optimality}]
\label{lemma: equivalence lemma}
    Let $\pi_\theta(a|x) = \frac{\exp\left(\Phi(x,a)^\top \theta \right)}{\sum_{b}\exp\left(\Phi(x,b)^\top \theta\right)}$. Also, let $\nu_\theta$ be the stationary distribution under policy $\pi_\theta$, and $A^{\pi}(x,a)$ be the advantage function under policy $\pi$ defined as $A^{\pi}(x,a) = q^{\pi}(x,a) - v^{\pi}(x)$. Then the update
    \begin{align*}
        \thetanew = \theta + \eta F_\theta^\dagger g_\theta
    \end{align*}
    with 
    \begin{align*}
         F_\theta &= \E_{x\sim \nu^{\pi_\theta}}\E_{a\sim \pi_\theta(\cdot|x)}\left[\nabla_\theta \log\pi_{\theta}(a|x) \nabla_\theta \log\pi_{\theta}(a|x)^\top \right] \\
        g_\theta &= \E_{x\sim \nu^{\pi_\theta}}\E_{a\sim \pi_\theta(\cdot|x)}\left[\nabla_\theta \log\pi_{\theta}(a|x) A^{\pi_\theta}(x,a) \right]
    \end{align*}
    implies: 
    \begin{align*}
         \pi_{\thetanew}(a|x)=\frac{\pi_\theta(a|x)\exp\left(\eta A^{\pi_\theta}(x,a)\right)}{Z_\theta(x)}
    \end{align*}
    where $Z_\theta(x)$ is a normalization factor that ensures $\sum_{a}\pi_{\thetanew}(a|x)=1$. 
\end{lemma}

To see this connection, notice that the update direction $w=F_\theta^\dagger g_\theta$ is the solution of 
\begin{align} 
   \min_{w\in\mathbb{R}^d} \E_{x\sim \nu^{\pi_\theta}}\E_{a\sim \pi_\theta(\cdot|x)} \left[\left\|w^\top \nabla_\theta\log\pi_\theta(a|x) - A^{\pi_\theta}(x,a)\right\|^2\right],    \label{eq: solution of}
\end{align}
and also by definition $\pi_{\thetanew}(a|x) = \frac{\exp\left(\Phi(x,a)^\top \thetanew \right)}{\sum_{b}\exp\left(\Phi(x,b)^\top \thetanew\right)}\propto \pi_\theta(a|x)\exp\left(\eta \Phi(x,a)^\top F_\theta^\dagger g_\theta\right)= \pi_\theta(a|x)\exp\left(\eta \Phi(x,a)^\top w\right)\propto  \pi_\theta(a|x)\exp\left(\eta \nabla_\theta \log \pi_\theta(a|x)^\top w\right)$. Therefore, if $w$ achieves a value of zero in \pref{eq: solution of}, we will have $\pi_{\thetanew}(a|x)\propto \pi_\theta(a|x) \exp\left(\eta A^{\pi_\theta}(x,a)\right)$. The proof of \cite{agarwal2019optimality} handles the general case where \pref{eq: solution of} is not necessarily zero. Notice that $\pi_\theta(a|x)\exp\left(\eta A^{\pi_\theta}(x,a)\right)$ is further proportional to $\pi_\theta(a|x)\exp\left(\eta q^{\pi_\theta}(x,a)\right)$, which is consistent with the intuition of our algorithm explained in \pref{sec: mdpexp2}. 

\subsection{Comparison between the NPG in \cite{agarwal2019optimality} and \algExp}
While the general formulations of the NPG in \cite{agarwal2019optimality} and \algExp are of the same form as shown by \pref{lemma: equivalence lemma} (apart from superficial differences, e.g., the average-reward setting versus the discounted setting), they use different ways to construct an estimator of $A^{\pi_\theta}(x,a)$ (or $q^{\pi_\theta}(x,a)$) when the learner does not have access to their true values and has to estimate them from sampling. We argue that under the setting considered in \pref{sec: mdpexp2}, our algorithm and analysis achieve the near-optimal regret of order $\otil(\sqrt{T})$ while theirs only obtains sub-optimal regret. 

In \algExp, we construct a nearly unbiased estimator of $w$ satisfying $q^{\pi_\theta}(x,a) + NJ^{\pi_\theta} = w^\top \Phi(x,a)$ (which exists under \pref{ass: linear policy value}), and feed it to the exponential weight algorithm. The way we do it is similar to how \exptwo constructs the reward estimators for adversarial linear bandits. In \algExp, to construct each estimator (denoted as $w_k$ there), the learner collects $\frac{B}{2N}=\otil(\frac{1}{\sigma})$ trajectories, with $\sigma$ defined in \pref{ass: assump A4}, and then aggregate them through a form of importance weighting introduced by $M_k^{-1}$. With this construction, $w_k^\top\Phi(x,a) $ has negligible bias (by \pref{lemma: wk and theta}) compared to $w^\top\Phi(x,a) $, while having variance upper bounded by a constant related to $\frac{1}{\sigma}$ (see the proof of \pref{lemma: perstate regret}). 

On the other hand, the estimator used in \cite{agarwal2019optimality} is an approximate solution of \pref{eq: solution of}. Under the same assumptions of \pref{ass: linear policy value} and \pref{ass: assump A4}, they use stochastic gradient descent to solve \pref{eq: solution of}, and obtain an estimator $\widehat{w}$ that makes $\widehat{w}^\top \nabla_\theta\log_\theta(a|x)$ $\epsilon$-close to $w^\top \nabla_\theta\log_\theta(a|x)$. To obtain such $\widehat{w}$, they need to sample $\order\left(\frac{1}{\epsilon^2}\right)$ trajectories. 

Comparing the two approaches, we see that to obtain a single estimator $\widehat{w}$ for the update direction $w=F_\theta^\dagger g_\theta$ in \pref{lemma: equivalence lemma}, \algExp algorithm calculates a nearly unbiased one with relatively high variance using a constant number of trajectories, while the NPG in \cite{agarwal2019optimality} calculates an $\epsilon$-accurate one with low variance using $\order\left(\frac{1}{\epsilon^2}\right)$ trajectories. The advantage of the former is that each estimator is cheaper to get, and the effect of high variance can be amortized over iterations. As shown in \pref{thm: MDPEXP2 bound}, $\algExp$ achieves $\otil(\sqrt{T})$ regret bound. On the other hand, to get an $\epsilon$-optimal policy, \cite{agarwal2019optimality} needs to use $\order\left(\frac{1}{\epsilon^2}\right)$ trajectories per iteration of policy update, and perform $\order\left(\frac{1}{\epsilon^2}\right)$ iterations of policy updates, leading to a total sample complexity bound of $\order\left(\frac{1}{\epsilon^4}\right)$. This translates to a regret bound of $\order(T^{\frac{3}{4}})$ in our setting at best. In fact, since the algorithms by \cite{abbasi2019politex} and \cite{hao2020provably} are also based on exponential weight, they also can be regarded as variants of NPG. However, the estimators they construct suffer the same issue described above, and can only get $\order(T^{\frac{3}{4}})$ or $\order(T^{\frac{2}{3}})$ regret.

We remark that the version of NPG by \cite{agarwal2019optimality} can also learn the optimal policy with a worse sample complexity $\order\left(\frac{1}{\epsilon^6}\right)$ under a weaker assumption compared to \pref{ass: assump A4} (which replaces $\frac{1}{\sigma}$ with the \emph{relative condition number} $\kappa$ defined in their Section 6.3). It is not clear how our approach can extend to this setting and obtain improved sample complexity. We leave this as a future direction. 